%% file: autoanchor.tex
\documentclass{article}

\PassOptionsToPackage{numbers, compress}{natbib}

 \usepackage[preprint]{neurips_2026}

\usepackage[utf8]{inputenc} 
\usepackage[T1]{fontenc}    
\usepackage{hyperref}       
\usepackage{url}            
\usepackage{booktabs}       
\usepackage{amsfonts}       
\usepackage{nicefrac}       
\usepackage{microtype}      
\usepackage{xcolor}         

\usepackage{graphicx}
\usepackage{xspace}
\usepackage{colortbl}
\usepackage{multirow}
\usepackage{makecell} 
\usepackage{rotating}
\usepackage{subcaption}
\usepackage{wrapfig}
\usepackage{arydshln}
\usepackage{float}
\usepackage{amsmath,amsthm,amssymb,bm}
\usepackage[normalem]{ulem}
\usepackage{enumitem}
\usepackage{tcolorbox}
\tcbuselibrary{skins,breakable}
\newtcolorbox{promptfancy}{
    enhanced,
    colback=white,
    colframe=orange!70!black,
    colbacktitle=orange!10,
    coltitle=black,
    attach title to upper,
    drop shadow,
    left=6pt, right=6pt,
    fontupper=\small\ttfamily,
    breakable,  
    overlay={
        \begin{tcbclipinterior}
        \draw[orange!20, line width=2mm] 
            (frame.south west) rectangle (frame.north east);
        \end{tcbclipinterior}
    }
}

\newtheorem{theorem}{Theorem}
\newtheorem*{theorem*}{Theorem}

\newtheorem{lemma}{Lemma}

\newtheorem{proposition}{Proposition}

\newtheorem{assumption}{Assumption}
\newtheorem*{assumption*}{Assumption}
\newtheorem{definition}{Definition}
\newtheorem*{definition*}{Definition}

\newcommand{\mysys}{AutoAnchor\space}
\newcommand{\mysysn}{AutoAnchor}

\title{\mysysn: Stable Diffusion Unlearning Using Cross-Attention as a Manifold Surrogate}

\author{%
  Siyuan Wen\textsuperscript{\rm 1}, Jiahao Zeng\textsuperscript{\rm 1}, Ningning Ding\textsuperscript{\rm 1}\thanks{Corresponding author.} \\
  \textsuperscript{\rm 1}Hong Kong University of Science and Technology (Guangzhou)\\
  \texttt{\{swen211,jzeng110\}@connect.hkust-gz.edu.cn}, 
  \texttt{ningningding@hkust-gz.edu.cn} \\
  }

\begin{document}

\maketitle

\input{src/abstract.tex}

\input{src/introduction.tex}
\input{src/background.tex}
\input{src/preliminary.tex}
\input{src/analysis.tex}
\input{src/autoanchor.tex}

\input{src/evaluation.tex}

\input{src/conclusion.tex}




\bibliography{ref}
\bibliographystyle{plain}

\newpage
\appendix

\input{src/appendix.tex}

\end{document}

%% file: src/abstract.tex
\begin{abstract}
Diffusion unlearning is essential for mitigating the generation of harmful or copyrighted content in text-to-image models. Current diffusion unlearning techniques determine the model update direction by either using alternatives of the target concept as an anchor or using empty prompts. The anchor-based method relies on manually and semantically-chosen anchors that risk biased unlearning, while the anchor-free method inherently suffers from unrobust unlearning due to unconstrained latent updates. In this work, we theoretically formalize such unstable diffusion unlearning issues under the manifold hypothesis and prove that lacking a manifold-proximal anchor inevitably induces significant normal-space drift that degrades unlearning performance. To achieve stable unlearning, we propose \mysysn, a two-stage framework that automatically synthesizes manifold-proximal anchors. However, direct geometric manifold optimization is computationally intractable. To address this challenge, \mysys introduces a novel cross-attention consistency loss which serves as a highly efficient surrogate of manifold proximity. Experimental results demonstrate that \mysys effectively achieves robust and unbiased unlearning across various state-of-the-art baselines, significantly improving targeted concept removal (by up to 31.04\% in CLIP score) and non-target utility (by up to 4.18\% in CLIP score). Moreover, \mysys can also be easily integrated into existing diffusion unlearning methods to enhance their unlearning performance (by 6.30\% for concept removal and 6.65\% for utility on average).
\end{abstract}

%% file: src/introduction.tex
\section{Introduction}\label{sec:intro}

Text-to-Image diffusion models have achieved remarkable success in generating high-fidelity images conditioned on input text prompts. However, their reliance on massive and uncurated web data introduces significant risks, \textit{e.g.}, generating harmful content or copyright infringement. To mitigate these risks, diffusion unlearning has emerged as a crucial technique. The goal of diffusion unlearning is to erase a diffusion model's ability to generate specific target concepts, which is usually realized by tuning the model parameters towards certain update directions.

Achieving effective erasure while preserving model utility is hard for diffusion unlearning. This is because the modern diffusion models operate in a high-dimensional latent space, but the valid and realistic images reside on a low-dimensional data manifold within that. Without a carefully designed update direction, diffusion unlearning can easily push the model off the data manifold when erasing target concepts, thereby degrading the erasure effectiveness and damaging the model's utility.

Regarding how to determine the update directions, current diffusion unlearning methods are categorized into \textit{anchor-free} unlearning and \textit{anchor-based} unlearning. \textit{Anchor-free} methods (\textit{e.g.,}~\cite{esd,CCRT,ResAlign}) utilize target-irrelevant directions for unlearning, simply pushing the generation away from the target concepts. For example, they tune the models toward valid but unconditional images when given prompts containing target concepts. While computationally straightforward, such update directions inherently cause off-manifold drift and lead to \textbf{unrobust unlearning}, manifesting as catastrophic forgetting~\cite{amnesia}, degraded image quality~\cite{salun,GLoCE}, and weakened text-image alignment~\cite{boostingSDU,CCRT} on non-target concepts. Moreover, such methods risk fragile forgetting that can be easily bypassed by adversarial prompts~\cite{adversarialSDU,codo,advUnlearn} or simple fine-tuning~\cite{catastrophicSDU}. 
Existing \textit{Anchor-based} methods (\textit{e.g.,}~\cite{conabl,spm,codo}) redirect the model from the target concept toward a specific alternative concept (the anchor), trying to stay on the manifold during unlearning. However, the anchors that they heavily rely on are all manually chosen based on semantic similarity, which is insufficient for ensuring manifold proximity and risks \textbf{biased unlearning}. With biased anchors, the unlearning can still suffer from the unrobust unlearning issues above and may further cause imprecise forgetting~\cite{eraseDiff}.

Both approaches suffer from a shared limitation: the inability to identify a manifold-proximal anchor for robust and unbiased unlearning. To fill this gap, we analyze the unlearning process under the manifold hypothesis.
Geometrically, valid images lie on a low-dimensional data manifold within the vast diffusion latent space, as illustrated in Figure \ref{fig:illu}. A robust and effective unlearning update should smoothly traverse along this manifold to alter target concepts, rather than pushing the representation away from it. Therefore, we decompose the unlearning update directions into their tangential (along the manifold) and normal (orthogonal to the manifold) components.
Our analysis demonstrates that lacking a manifold-proximal anchor (\textit{e.g.}, anchor-free methods and unlearning via bias anchors) fundamentally leads to normal-space drift, forcing the generation trajectory off the manifold. Moreover, we prove that manifold-proximal anchors effectively align unlearning directions closely with the local tangent space of the data manifold, promoting robust and unbiased unlearning.

\begin{figure}
    \centering
    \includegraphics[width=0.8\textwidth]{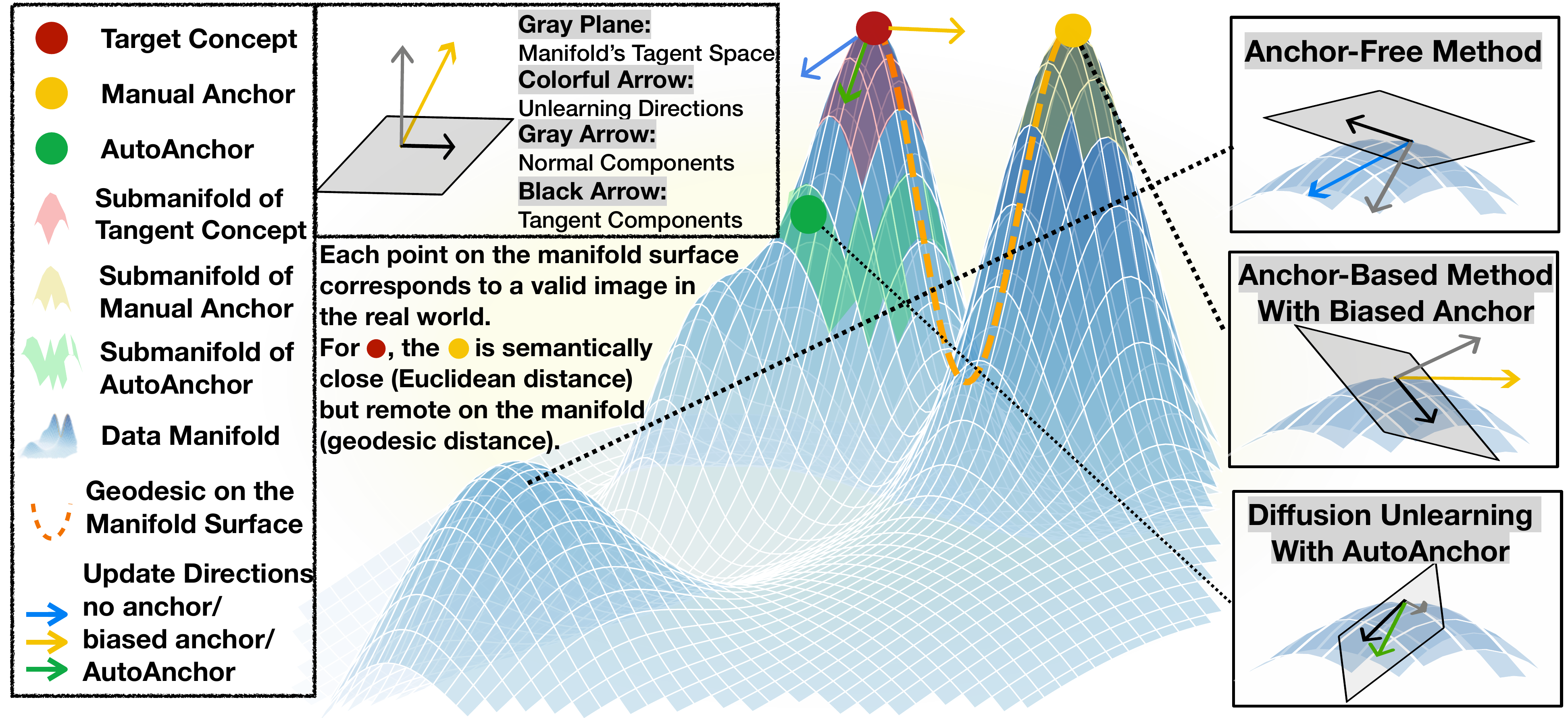}
    \caption{Illustration of the data manifold and unlearning update directions. Without a manifold-proximal anchor (\textit{e.g.}, our \mysysn), the update direction has a large normal component that pushes the model off the manifold, leading to unrobust or biased unlearning.}\label{fig:illu}
    \vspace{-10 pt}
\end{figure}

However, constructing such a manifold-proximal anchor presents non-trivial challenges. Direct optimization on the data manifold is computationally intractable~\cite{hardmanifold2,hardmanifold3,hardmanifold1}, as the manifold is only implicitly defined by the complex distribution of the diffusion model's latent space. Furthermore, there exists a fundamental domain gap: while the unlearning update operates in the diffusion latent space, the anchor is optimized within the vast prompt embedding space. To bridge this gap, our key insight is that in diffusion architectures, a frozen model's cross-attention maps explicitly dictate the layout of certain concepts when generating images, which implicitly reflect the data manifold learned by the model. By leveraging additional semantic constraints, these attention maps serve as a powerful and efficient proxy for tangential manifold proximity (as we prove in Appendix \ref{seca:theoretical_bridge}), enabling us to optimize an anchor that aligns its attention patterns with those of the target concept. 

To this end, we propose \mysysn, a lightweight two-stage process that automatically synthesizes manifold-proximal anchors for diffusion unlearning. It first filters and aggregates valid candidate concepts automatically, and then optimizes the anchors to be manifold-proximal to the target concept. We further propose a cross-attention consistency loss, a novel manifold optimization surrogate specific to anchor construction. \mysys empirically mitigates off-manifold drift and thus effectively achieves robust and unbiased diffusion unlearning. Moreover, it can be integrated into existing diffusion unlearning methods to further improve their unlearning performance.

Our main contributions are summarized as follows:
\begin{itemize}[leftmargin=10pt]
    \item \textbf{Theoretical Formalization of Diffusion Unlearning Issues:} We analyze diffusion unlearning via the manifold hypothesis and theoretically prove that lacking a manifold-proximal anchor inevitably causes normal-space drift, providing a mathematical explanation for why existing anchor-free and manual anchor-based methods fundamentally suffer from unstable diffusion unlearning.
    \item \textbf{Novel Manifold-Proximal Surrogate and Unlearning Framework:} To overcome the computational intractability and domain gap of direct manifold optimization, we introduce a cross-attention consistency loss. Using this surrogate loss, our proposed \mysys efficiently and automatically constructs manifold-proximal anchors for achieving unbiased and robust diffusion unlearning. Furthermore, we theoretically explain how \mysys approximates manifold proximity. 
    \item \textbf{Empirical Evaluation:} Extensive experiments show that \mysys improves the diffusion unlearning performance of the state-of-the-art baselines in both concept removal (improving 6.30\% in CLIP score on average) and model utility (improving 6.65\% in CLIP score on average). Moreover, ablation studies show the effects of cross-attention loss and its robustness on hyperparameters.
\end{itemize}

%% file: src/background.tex
\section{Related Works}
Based on how to determine the update direction, existing diffusion unlearning methods are categorized into anchor-free and anchor-based approaches.
\textbf{Anchor-free} approaches determine their update directions using coarse signals like empty prompts, generic concepts, or unconditional scores. For example, ESD~\cite{esd} utilizes the difference of predictions between the target concept and empty prompt as the update direction. Such imprecise directions cause a series of robustness issues in diffusion unlearning (unrobust against adversarial prompts, \textit{etc}). Although existing works have made efforts to address them, they target individual issues in isolation. Specifically, \cite{advUnlearn} only focuses on robustness against adversarial prompts; \cite{boostingSDU,amnesia,tuneSD,overunlearn,salun} only aim to strengthen model utility; \cite{mace} is solely concerned with utility loss after multiple unlearning; and \cite{catastrophicSDU,unstableSDU,SemanticSurgery} only pursue effective concept removal. No single existing work solves these issues simultaneously, due to ignoring the root problem of imprecise update directions. Our method fundamentally fills this gap by generating a stable anchor for precise unlearning direction, thereby enabling robust diffusion unlearning.

\textbf{Anchor-based} approaches pre-define specific anchor concepts for the target concepts, aiming to provide more precise unlearning directions than anchor-free methods. For example, ConAbl~\cite{conabl} first uses the generation direction differences between the target and anchor concepts as the update direction, thus realizing better utility preservation than ESD. Recent anchor-based methods further explore more sophisticated problems like reliable erasure~\cite{allbut1,codo,ResAlign,eraseDiff}, multiple unlearning~\cite{massSDU,CCRT}, training-free unlearning~\cite{CURE}, concept erosion~\cite{spm,DUO,FADE,GLoCE}, and attention-based unlearning~\cite{Forget-me-not,conceptreplacer}. However, existing anchor-based methods heavily rely on the assumption that their manually chosen anchors are precise. This assumption is not practical when semantically close anchors are not manifold-proximal to the target concepts, leading to biased unlearning as we show in Appendix \S\ref{seca:bias}. While some works~\cite{eraseDiff,spm,misdirection,cooccurring} attempt to generate anchors automatically, they still rely on semantic closeness and thus fail to provide fundamentally unbiased anchors.
Our method leverages cross-attention maps as a practical manifold surrogate that manages to construct manifold-proximal anchors, thereby achieving unbiased unlearning. Moreover, our method can be easily integrated into existing diffusion unlearning methods to further improve their unlearning performance.

%% file: src/preliminary.tex
\section{Preliminaries}\label{sec:preliminaries}

\textbf{Latent Diffusion Models (LDMs).} Modern text-to-image models (\textit{e.g.}, Stable Diffusion~\cite{ldm}) operate in a compressed latent space. Given an input text prompt, the model first maps it into the corresponding embedding $\mathbf{e} \in \mathcal{E}$. Next, the model predicts the latent noise $\epsilon \in \mathbb{R}^d$, which is then used to generate an image latent $\mathbf{z} \in \mathbb{R}^d$ as the output. Training LDMs involves two processes. First, in the noise diffusion process, Gaussian noise $\epsilon$ is progressively added to the latent state $\mathbf{z}_0$ of a valid image, producing a noisy latent $\mathbf{z}_t$ at timestep $t \sim \mathcal{U}[0,T]$, where $\mathcal{U}$ denotes the uniform distribution. Second, in the denoising process, the diffusion model, a neural network parameterized by $\boldsymbol{\epsilon}_{\theta}: \mathbb{R}^d \times \mathcal{E} \times [0,T] \to \mathbb{R}^d$, predicts the added noise and learns the image manifold by minimizing the standard objective $\mathcal{L}_{\rm SD}(\theta) = \mathbb{E}_{\mathbf{z}_0, \mathbf{e}, t, \boldsymbol{\epsilon}} \left[ \|\boldsymbol{\epsilon} - \boldsymbol{\epsilon}_\theta(\mathbf{z}_t, \mathbf{e}, t)\|^2 \right]$. LDMs use the predicted noises to generate an image $\mathbf{z}_0$ by progressively denoising $\mathbf{z}_{\rm t-1} = \mathbf{z}_{\rm t} - \boldsymbol{\epsilon}_\theta(\mathbf{z}_t, \mathbf{e}, t)$, for $t=T,...,2,1$.

\textbf{Diffusion Unlearning.} Denote a pre-trained diffusion model $\boldsymbol{\epsilon}_{\theta^*}$ and a target concept $c_u \in \mathcal{C}$ to unlearn (e.g., ``Mickey Mouse''). The goal of diffusion unlearning is to obtain a tuned model $\boldsymbol{\epsilon}_{\theta^u}$ that is incapable of generating images related to the target concept $c_u$ when prompted, while preserving its generation capabilities for all non-target concepts. Existing unlearning algorithms achieve this by tuning the model's parameters along a specific update direction to alter its prediction for $c_u$. To formulate this update direction, we recall the known fact~\cite{ddpm,score} that diffusion models learn the score function (\textit{i.e.}, gradient of the log probability) of the data distribution, \textit{i.e.}, $\boldsymbol{\epsilon}_\theta(\mathbf{z}_t, \mathbf{e}, t) \propto \nabla_{\mathbf{z}_t} \log p_t(\mathbf{z}_t | \mathbf{e})$. This equivalence allows us to represent unlearning updates (on the data manifold) as noise prediction differences (in latent space) under two distinct embeddings. For any \textit{anchor-free} method ($\Delta^{\text{AF}}$) and any \textit{anchor-based} method ($\Delta^{\text{AB}}$), these update directions are expressed as:
\begin{equation}\label{eq:uudAF}
    \Delta^{\text{AF}} = \boldsymbol{\epsilon}_{\theta^*}(\mathbf{z}_t,\mathbf{e}_{u}, t) - \boldsymbol{\epsilon}_{\theta^*}(\mathbf{z}_t, \mathbf{e}_\emptyset, t), \quad \Delta^{\text{AB}} = \boldsymbol{\epsilon}_{\theta^*}(\mathbf{z}_t, \mathbf{e}_{u}, t) - \boldsymbol{\epsilon}_{\theta^*}(\mathbf{z}_t, \mathbf{e}_a, t),
\end{equation}
where $\mathbf{e}_{u}$, $\mathbf{e}_a$, and $\mathbf{e}_\emptyset$ denote the embeddings for the target concept $c_u$, the anchor $c_a\in \mathcal{C}$, and the empty prompt, respectively. The geometric properties of the conditioning embedding (empty or anchor) critically determine unlearning stability, as we will formally analyze in \S\ref{sec:analysis}.

%% file: src/analysis.tex
\section{Theoretical Analysis of Unstable Diffusion Unlearning}\label{sec:analysis}

In \S\ref{sec:the_notations}, we first introduce the manifold notations and statistical assumptions. Next we theoretically analyze why existing diffusion unlearning methods are unstable in \S\ref{sec:the_anchorfree}. In \S\ref{sec:the_anchorbased}, we further prove that unlearning updates determined by {manifold-proximal} anchors successfully mitigate this instability. Complete proofs, detailed assumption justifications, and extra stability analyses are in Appendix \ref{seca:analysis}.

\subsection{Problem Setup and Statistical Assumptions}\label{sec:the_notations}
We consider a smooth and compact manifold embedded in the diffusion latent space $\mathcal{M} \subset \mathbb{R}^d$ which is supported by all valid images. All data points associated with the target concept $c_u$ support a $d_u^*$-dimensional submanifold $\mathcal{M}_u \subset \mathcal{M}$. To analyze the stability of unlearning updates, we decompose the gradient directions relative to the concept manifold $\mathcal{M}_u$. For any latent state $\mathbf{z}_u \in \mathcal{M}$, $T_{u}$ denotes the $d_u^*$-dimensional \textit{tangent space} (directions along the manifold surface), while $N_{u}$ denotes the $(d - d_u^*)$-dimensional \textit{normal space} (directions orthogonal to the manifold). 
Fundamentally, effective unlearning requires erasing the target concept while maintaining the generation on the manifold to preserve utility. Therefore, a desired update direction $\Delta$ should be tightly aligned with the tangent space, \textit{i.e.}, $\Pi_{N_u} \Delta \approx 0$ where $\Pi$ denotes projection operation. 

Next we will introduce some standard assumptions on data distribution and LDMs that are widely adopted in existing literature. Assumptions~\ref{assumption:high_dim} and \ref{assumption:isotropy} directly follow from the manifold hypothesis and are supported by empirical studies of intrinsic dimension and high-dimensional geometry~\cite{intrinsicdim1,isobehavior,highdimProb}. Assumption~\ref{assumption:error_structure} formalizes how score approximation errors decompose along the data manifold~\cite{decompose1,scoredecomposeerror2}.
\begin{assumption}[High-Dimensional Regime]\label{assumption:high_dim}
The dimension of ambient latent space $\mathbb{R}^d$ is much larger than the dimension of the data manifold $\mathcal{M}$, \textit{i.e.}, $d_u^* / d \to 0$. 
\end{assumption}
\begin{assumption}[Score Approximation Error]\label{assumption:error_structure}
The score approximation error $\varepsilon(\mathbf{z}_t, t) = \boldsymbol{\epsilon}_{\theta}(\mathbf{z}_t, t, \mathbf{e}) - \nabla_{\mathbf{z}_t} \log p_t(\mathbf{z}_t | \mathbf{e})$ has zero mean, and its variance projected onto the normal space is bounded by $\sigma_{N}^2$.
\end{assumption}
\begin{assumption}[Isotropy of Unconditional/Biased Interference]\label{assumption:isotropy}
For anchor-free unlearning, the normalized update vector $v_{\text{AF}} = \Delta^{\text{AF}}/\|\Delta^{\text{AF}}\|$ is approximately isotropic in high dimensions, and the same applies to anchor-based unlearning $v_{\text{AB}} = \Delta^{\text{AB}}/\|\Delta^{\text{AB}}\|$ with biased anchor $a\nsubseteq \mathcal{M}_u$.
\end{assumption}
We further elaborate on the justifications and details of these assumptions in Appendix \ref{seca:Justifications}.

\subsection{The Inevitability of Unstable Diffusion Unlearning for Existing Methods}\label{sec:the_anchorfree}
Under Assumptions \ref{assumption:high_dim}-\ref{assumption:isotropy}, we proves existing diffusion unlearning methods' inevitable normal drift.
\begin{theorem}[Inevitability of Normal Drift]
\label{thm:anchor_free_main}
For anchor-free unlearning and biased anchor-based unlearning, the fraction of the update energy projected into the normal direction is concentrated:
\begin{equation}\label{eq:theorem1}
    \forall \delta \in [0,1), \quad \mathbb{P}\left(\frac{\|\Delta_N\|^2}{\|\Delta\|^2} > 1 - \frac{d_u^*}{d} - \delta\right) \geq 1- \exp\left(-\frac{d\delta^2}{8}\right).
\end{equation}
\end{theorem}
Equation~\ref{eq:theorem1} means that with a probability larger than or equal to $1-\exp(-{d\delta^2}/{8})$, the fraction of update energy in the normal direction (\textit{i.e.}, $\|{\Delta_N\|^2}/{\|\Delta\|^2}$) is at least $1-d^*_u/d-\delta$. Note that as the latent space dimension $d$ is very large,  $1-d^*_u/d-\delta$ means a large normal drift, and $1-\exp(-{d\delta^2}/{8})$ is close to 1, which is an overwhelming probability. 
Therefore, nearly all unlearning energy inevitably leaks into the normal space and pushes the model off the manifold, mechanically explaining why existing diffusion unlearning risks being unrobust and biased. We provide the absolute magnitude bound for $\|\Delta_N\|$ in Appendix \ref{seca:anchor_free_abs} (Theorem \ref{thm:anchor_free_full}), which further demonstrates this instability.

\subsection{Stabilizing Unlearning via Manifold-Proximal Anchors}\label{sec:the_anchorbased}

To prove the potential of stable unlearning, we first define a $\zeta$-proximal anchor based on the the update direction it induces.
For all latent points related to the target concept, an anchor is $\zeta$-proximal anchor if the update direction from this anchor to each latent point is predominantly along the data manifold (with no larger than $\zeta^2$ fraction of update energy in the normal direction). That is, a smaller $\zeta$ means more manifold-proximal (see the formal Definition \ref{def:proximal_anchor} in Appendix \ref{seca:setup}). Therefore, we have:

\begin{theorem}[Suppression of Normal Drift via Proximal Anchors]
\label{thm:proximal_anchor_main}
Under Assumptions \ref{assumption:high_dim} and \ref{assumption:error_structure}, for diffusion unlearning with a $\zeta$-proximal anchor, the normal component $\Delta_N^{\text{PA}} = \Pi_N(\Delta^{\text{AB}})$ satisfies:
\begin{equation}
    \mathbb{P}\left(\frac{\|\Delta_N^{\text{PA}}\|^2}{\|\Delta^{\text{PA}}\|^2} \geq \zeta^2 + \frac{d_u^*}{d} + \delta\right) \leq \exp\left(-\frac{d\delta^2}{8}\right).
\end{equation}
\end{theorem}

Theorem \ref{thm:proximal_anchor_main} shows that given the large $\delta$, with a very low probability (\textit{i.e.}, no larger than $\exp(-{d\delta^2}/{8})$), the normal energy fraction is larger than a small number $\zeta^2 + d_u^*/d + \delta$. This reveals that if a tightly bounded $\zeta$-proximal anchor is utilized, the unlearning will leak almost no energy in the normal direction, being potential to stable unlearning. We also provide the absolute magnitude bound for $\|\Delta_N^{\text{PA}}\|$ in Appendix \ref{seca:anchor_based_abs} (Theorem \ref{thm:proximal_anchor_full}), which proves the stability as well. Moreover, as proven via the full random-walk analysis in Appendix \ref{seca:analysis} (Theorem \ref{thm:stability_comparison}), this suppression grants proximal anchor-based methods an exponential stability advantage over anchor-free methods. In \S\ref{sec:autoanchor}, we introduce \mysys to construct a manifold-proximal anchor in diffusion unlearning.

%% file: src/autoanchor.tex
\section{Design of \mysysn}\label{sec:autoanchor}

To fundamentally avoid the instability and bias problems in diffusion unlearning, the anchor generated by \mysys must satisfy three key requirements. \textbf{(R1) Non-Synonymous:} It must not be a direct synonym of the target concept. For example, using ``naked'' as an anchor fails to unlearn the "nude" concept, because the two distributions largely overlap, leaving the target insufficiently suppressed. \textbf{(R2) Semantic Alternative:} It should be a valid semantic alternative related to the target concept. For example, using an irrelevant concept ``Mickey Mouse'' as an anchor to unlearn the ``nude'' concept offers no meaningful guidance and effectively behaves as unconditional noise. \textbf{(R3) Manifold-Proximal:} It should be manifold-proximal to the target concept to guide a stable unlearning direction, as demonstrated by our theoretical analysis in \S\ref{sec:analysis}. 
In addition to the three requirements of the anchor above, the generation process itself must be \textbf{automatic} (to avoid manually introducing extra subjective bias and ensure scalability of unlearning) and \textbf{efficient} (acting as a lightweight preprocessing step).

To meet these criteria, we propose a two-stage framework, illustrated in Figure~\ref{fig:overview}. The first stage initializes the anchor by clustering and aggregating a diverse set of candidate concepts (\S\ref{sec:init}), explicitly enforcing \textbf{R1} and \textbf{R2}. The second stage optimizes the initial anchor to approach manifold proximity using a cross-attention surrogate (\S\ref{sec:opt}), thereby satisfying \textbf{R3}. Finally, we apply \mysys to unlearn the target concept from the model, achieving robust and unbiased unlearning.

\begin{figure}
    \centering
    \includegraphics[width=0.76\textwidth]{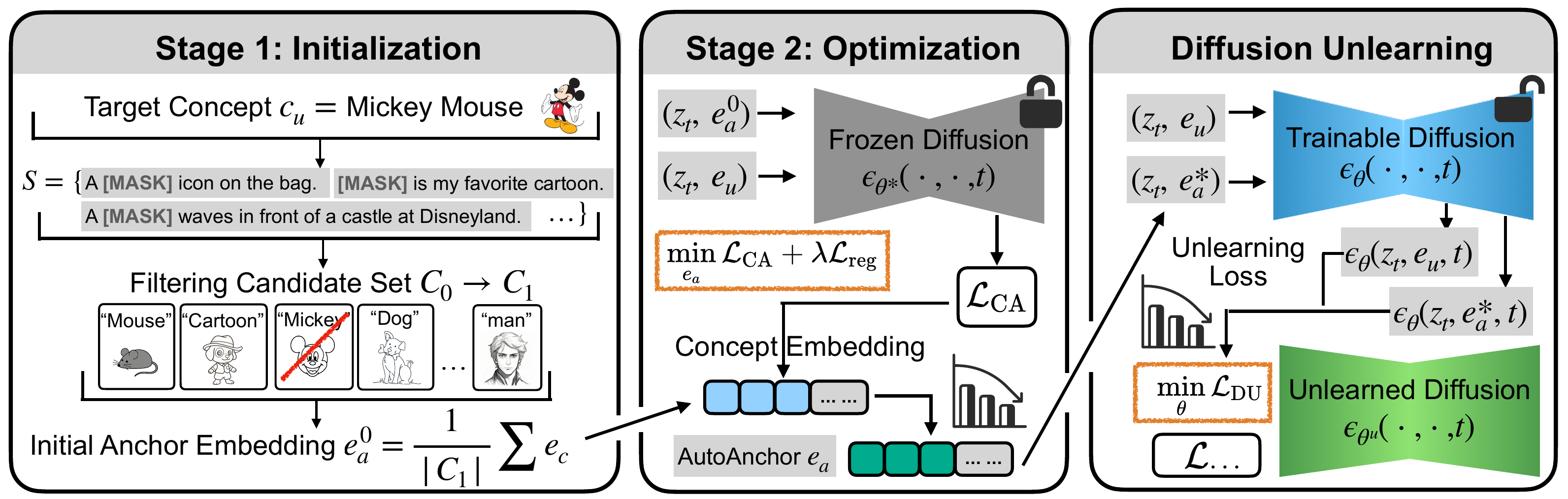}
    \caption{Illustration of \mysysn. Given a target concept $c_u$ to unlearn, \mysys first generates, filters, and aggregates the candidate concepts. Then \mysys optimizes the cross-attention consistency loss $\mathcal{L}_{\rm CA}$ for \textbf{R3} and a regularization term $\mathcal{L}_{\rm reg}$ to maintain \textbf{R1} and \textbf{R2}. Finally, with \mysysn, we tune the model parameters to achieve robust and unbiased diffusion unlearning.}\label{fig:overview}
\end{figure}

\subsection{Stage \MakeUppercase{\romannumeral 1}: Efficient Anchor Initialization}\label{sec:init}

To initialize anchor embedding $\mathbf{e}_a^0$ that satisfies \textbf{R1} and \textbf{R2}, we run the following two steps. 

\textbf{(\MakeLowercase{\romannumeral 1}) Automatic Candidate Generation.} To generate a diverse set of candidate concepts $C_0$ that act as valid alternatives to the target concept $c_u\in\mathcal{C}$ (\textbf{R2}), we rely on the data distribution $P$ (realized via an LLM) within the prompt space $\mathcal{C}$. Specifically, we first use LLM to generate a set of $N$ sentences $\mathcal{S}=\{\mathbf{y}^{(1)}, \dots, \mathbf{y}^{(N)}\}$ containing $c_u$. For each sentence $\mathbf{y}^{(i)}$, we replace $c_u$ with a \texttt{[MASK]} token, denoting the masked sequence as $\tilde{\mathbf{y}}^{(i)}$. We then predict the most likely substitute token $\hat{c}^{(i)}$ for each sequence. Finally, we calculate the frequency of each predicted token across all $N$ sentences, selecting the top-$k$ tokens as our initial candidate set:
    $C_0 = {\text{TopK}}_{c\in\mathcal{C}} \sum_{i=1}^N \mathbb{I} \left( \hat{c}^{(i)} = c \right)$, where $\hat{c}^{(i)} = {\arg\max}_{c'\in\mathcal{C}} \, P(c' \mid \tilde{\mathbf{y}}^{(i)})$ and $\mathbb{I}(\cdot)$ is the indicator function.

\textbf{(\MakeLowercase{\romannumeral 2}) Candidate Filtering and Aggregation.} To prevent the inclusion of synonyms (\textbf{R1}), we filter $C_0$ in the continuous embedding space. We extract the text embeddings for all concepts in $C_0 \cup \{c_u\}$ and partition them into $M$ clusters using a clustering function $f_{\text{cluster}}$. To guarantee semantic distinction from the target, we discard any candidate assigned to the same cluster as $c_u$. The final filtered set $C_1$ is used to compute the initial anchor embedding $\mathbf{e}_a^0$ via a simple average:
\begin{equation}
    \mathbf{e}_a^0 = \frac{1}{|C_1|} \sum_{c \in C_1} \mathbf{e}_c, \quad \text{where} \quad C_1 = \left\{ c \in C_0 \mid f_{\text{cluster}}(\mathbf{e}_c) \neq f_{\text{cluster}}(\mathbf{e}_{u}) \right\}.
\end{equation}
This process initializes a robust starting point $\mathbf{e}_a^0$, which is efficient and entirely bypasses manual human selection as it only requires standard LLM forward passes and lightweight clustering.

\subsection{Stage \MakeUppercase{\romannumeral 2}: Manifold Optimization via Cross-Attention Surrogate}\label{sec:opt}
To satisfy \textbf{R3}, we must optimize the initial anchor $\mathbf{e}_a^0$ to approach the target concept $c_u$ on the data manifold. As direct geometric manifold optimization in a high-dimensional latent space is computationally intractable, we propose a cross-attention consistency loss as an efficient surrogate:
\begin{equation}
    \mathcal{L}_{\text{CA}}(\mathbf{e}_a) = \mathbb{E}_{t \sim \mathcal{U}, \mathbf{z}_t \sim p_t(\cdot|\mathbf{e}_{u})} \left[ \sum_{l} \left\| A_l(\mathbf{z}_t, t, \mathbf{e}_{u}) - A_l(\mathbf{z}_t, t, \mathbf{e}_a) \right\|_F^2 \right],
\end{equation}
where $A_l(\cdot)$ denotes the cross-attention map of the frozen U-Net at layer $l$ and timestep $t$, conditioned on the latent $\mathbf{z}_t$ and embedding $\mathbf{e}$. Cross-attention maps help the U-Net to consider text embeddings, which capture the spatial and structural grounding of concepts during generation. If an anchor is already a semantic alternative (\textbf{R1} \& \textbf{R2}), such cross-attention alignment will force the anchor to induce similar structural patterns, acting as a strong proxy for approaching manifold proximity (\textbf{R3}).

To strictly preserve \textbf{R1} and \textbf{R2} during optimization, we restrict the anchor to a valid semantic subspace. We define a spherical shell bounding the target embedding: $\mathcal{E}_a = \{\mathbf{e} \in \mathcal{C} \mid r_{-} \leq \|\mathbf{e} - \mathbf{e}_{u}\| \leq r_{+}\}$, where the inner radius $r_{-}$ and outer radius $r_{+}$ are determined by the Euclidean distances from $\mathbf{e}_{u}$ to its closest and farthest cluster centroids in stage \MakeUppercase{\romannumeral 1}, respectively. This boundary ensures the effects of filtering in Stage \MakeUppercase{\romannumeral 1}, so we enforce it during optimization using a regularization penalty $\mathcal{L}_{\text{reg}}$. Denoting $\lambda$ as the regularization strength, we have the optimization objective for \mysysn as follows:
\begin{equation}\label{eq:reg}
    \min_{\mathbf{e}_a} \mathcal{L}_{\text{CA}} + \lambda \mathcal{L}_{\text{reg}}, \text{ where } \mathcal{L}_{\text{reg}}(\mathbf{e}_a) =  \text{ReLU} \left( r_{-} - \|\mathbf{e}_a - \mathbf{e}_{u}\| \right) + \text{ReLU} \left( \|\mathbf{e}_a - \mathbf{e}_{u}\| - r_{+} \right).
\end{equation}

With the generated anchor embedding $\mathbf{e}_a^*$ from (\ref{eq:reg}), we minimize the following diffusion unlearning objective $\mathcal{L}_{\rm DU}(\theta)=\mathcal{L}_{\rm rm}(\theta) + \alpha \mathbb{E}_{\mathbf{z}_t, t} \left[ \|\boldsymbol{\epsilon}_{\theta^*}(\mathbf{z}_t, t) - \boldsymbol{\epsilon}_{\theta}(\mathbf{z}_t, t)\|^2 \right]$, where $\mathcal{L}_{\rm rm}(\theta)=\|\epsilon_{\theta^*}(\mathbf{z}_t, \mathbf{e}_{a}^*, t) - \epsilon_{\theta}(\mathbf{z}_t, \mathbf{e}_{u}, t)\|$ is a widely accepted unlearning loss in previous works~\cite{conabl,spm,codo}, and the second term is a common regularization to preserve model utility~\cite{tuneSD,CCRT}. We can also easily adapt \mysys to other unlearning losses by replacing $\mathbf{e}_a^*$ in their original loss formulations $\mathcal{L}_{\rm rm}$, which can also improve the unlearning performance. 
We theoretically explain how optimizing Equation (\ref{eq:reg}) enables \mysys to approximate manifold proximity in Appendix \ref{seca:theoretical_bridge}.

%% file: src/evaluation.tex
\section{Evaluation}\label{sec:exp}

\textbf{Baselines \& Diffusion Models.} We compare \mysys against seven state-of-the-art diffusion unlearning baselines: four anchor-free methods (ESD \cite{esd}, TuneSD-AF (anchor-free) \cite{tuneSD}, ResAlign \cite{ResAlign}, and CCRT \cite{CCRT}) and three anchor-based methods (ConAbl \cite{conabl}, FMN \cite{Forget-me-not} and TuneSD-AB (anchor-based) \cite{tuneSD}). These baselines were primarily built on SD v1.4 and focused on erasing specific concept categories, including copyrighted characters (e.g., Mickey Mouse), artistic styles (e.g., Van Gogh), and unsafe content (e.g., nudity). We adopt the same settings to ensure a fair comparison.

\textbf{Evaluation Metrics.} Following standard evaluation protocols, we assess model performance using CLIP scores and the Fréchet Inception Distance (FID). A higher CLIP score indicates better text-image semantic alignment, while a lower FID indicates superior image quality. We evaluate the models based on two primary criteria: (1) \textbf{Removal}: The model should fail to generate the erased concept, resulting in a low CLIP score on target concept prompts. (2) \textbf{Utility Preservation}: The model must retain its general generative capabilities on non-target concepts, which requires maintaining high text-image alignment (a high CLIP score) alongside high image quality (a low FID score).

In the following sections, we first present the main experiment results in \S\ref{sec:expmain}, demonstrating \mysysn's superior unlearning performance against state-of-the-art baselines. In \S\ref{sec:expabl}, we conduct an ablation study to validate the necessity of our two-stage design. Next, in \S\ref{sec:expeffi}, we quantitatively confirm that \mysys is highly efficient and practical for real-world deployment.

\begin{table*}[t]
  \centering
  \caption{Quantitative evaluation of \mysys across other unlearning methods. ``Orig'' shows the baseline methods, and ``+'' indicates baselines integrated with \mysysn. The $\Delta$ columns show the relative performance change of baselines caused by AutoAnchor integration, with green/red indicating improvement/degradation. For $\Delta$ values $\le2\%$, they fall below the error precision threshold and are marked in gray. We \textbf{bold} the best results and \underline{underline} the second best results to highlight them.}\label{tab:main_results}
  \resizebox{0.86\linewidth}{!}{%
  \begin{tabular}{l ccc ccc ccc}
    \toprule
    \multirow{2}{*}{\textbf{Method}} & \multicolumn{3}{c}{\textbf{CLIP}(U) $\downarrow$} & \multicolumn{3}{c}{\textbf{CLIP}(R) $\uparrow$} & \multicolumn{3}{c}{\textbf{FID} $\downarrow$} \\ 
    \cmidrule(lr){2-4} \cmidrule(lr){5-7} \cmidrule(lr){8-10}
    & \textbf{Orig} & \textbf{+\mysys} & \textbf{$\Delta$(\%)} & \textbf{Orig} & \textbf{+\mysys} & \textbf{$\Delta$(\%)} & \textbf{Orig} & \textbf{+\mysys} & \textbf{$\Delta$(\%)} \\
    \midrule
    \multicolumn{10}{c}{\textbf{Target Concept: Mickey Mouse}} \\
    \midrule
    ESD \cite{esd} & 27.22 & 26.21 & {\color[RGB]{0,120,0}3.71 $\downarrow$} & 27.65 & 29.53 & {\color[RGB]{0,120,0}6.80 $\uparrow$} & 153.06 & \textbf{148.75} & {\color[RGB]{0,120,0}2.82 $\downarrow$} \\
    tuneSD-AF \cite{tuneSD} & 35.05 & 34.87 & {\color[RGB]{128,128,128}0.51 $\downarrow$} & {34.55} & \textbf{34.86} & {\color[RGB]{128,128,128}0.90 $\uparrow$} & 185.96 & 183.43 & {\color[RGB]{128,128,128}1.36 $\downarrow$} \\
    ResAlign \cite{ResAlign} & 28.74 & 28.71 & {\color[RGB]{128,128,128}0.10 $\downarrow$} & 29.39 & 31.79 & {\color[RGB]{0,120,0}8.17 $\uparrow$} & 156.25 & 157.27 & {\color[RGB]{128,128,128}0.65 $\uparrow$} \\
    CCRT \cite{CCRT} & 26.04 & \underline{18.75} & {\color[RGB]{0,120,0}28.00 $\downarrow$} & 27.09 & 28.22 & {\color[RGB]{0,120,0}4.17 $\uparrow$} & 152.94 & 152.89 & {\color[RGB]{128,128,128}0.03 $\downarrow$} \\
    ConAbl \cite{conabl} & 28.54 & 26.12 & {\color[RGB]{0,120,0}8.48 $\downarrow$} & 29.37 & 31.61 & {\color[RGB]{0,120,0}7.63 $\uparrow$} & 154.93 & 156.21 & {\color[RGB]{128,128,128}0.83 $\uparrow$} \\
    FMN \cite{Forget-me-not} & 32.45 & 31.91 & {\color[RGB]{128,128,128}1.66 $\downarrow$} & 31.78 & 34.72 & {\color[RGB]{0,120,0}9.25 $\uparrow$}  & 172.68 & 172.36 & {\color[RGB]{128,128,128}0.19 $\downarrow$} \\
    tuneSD-AB \cite{tuneSD} & 28.86 & 28.38 & {\color[RGB]{128,128,128}1.66 $\downarrow$} & 29.25 & 32.07 & {\color[RGB]{0,120,0}9.64 $\uparrow$} & 158.45 & 156.23 & {\color[RGB]{128,128,128}1.40 $\downarrow$} \\
    \textbf{\mysys (Ours)} & \textbf{18.77} & - & - & \underline{34.82} & - & - & \underline{152.74} & - & - \\
    \textbf{Average} & - & - & {\color[RGB]{0,120,0}6.30 $\downarrow$} & - & - & {\color[RGB]{0,120,0}6.65 $\uparrow$} & - & - & {\color[RGB]{128,128,128}0.62 $\downarrow$} \\
    \midrule
    \multicolumn{10}{c}{\textbf{Target Concept: Van Gogh Style}} \\
    \midrule
    ESD \cite{esd} & 24.94 & 22.00 & {\color[RGB]{0,120,0}11.79 $\downarrow$} & 31.28 & 33.28 & {\color[RGB]{0,120,0}6.39 $\uparrow$} & 162.04 & 164.54 & {\color[RGB]{128,128,128}1.54 $\uparrow$} \\
    tuneSD-AF \cite{tuneSD} & 34.00 & 33.70 & {\color[RGB]{128,128,128}0.88 $\downarrow$} & 33.52 & 33.66 & {\color[RGB]{128,128,128}0.43 $\uparrow$} & 232.41 & 229.05 & {\color[RGB]{128,128,128}1.45 $\downarrow$} \\
    ResAlign \cite{ResAlign} & 32.27 & 29.85 & {\color[RGB]{0,120,0}7.50 $\downarrow$} & 32.95 & 33.86 & {\color[RGB]{0,120,0}2.76 $\uparrow$} & 177.05 & 175.63 & {\color[RGB]{128,128,128}0.80 $\downarrow$} \\
    CCRT \cite{CCRT} & 19.33 & \underline{19.05} & {\color[RGB]{128,128,128}1.45 $\downarrow$} & 31.14 & 34.20 & {\color[RGB]{0,120,0}9.83 $\uparrow$} & 163.79 & 164.94 & {\color[RGB]{128,128,128}0.70 $\uparrow$} \\
    ConAbl \cite{conabl} & 30.85 & 25.02 & {\color[RGB]{0,120,0}18.90 $\downarrow$} & 32.69 & \textbf{34.98} & {\color[RGB]{0,120,0}7.01 $\uparrow$} & 169.54 & \underline{161.13} & {\color[RGB]{0,120,0}4.90 $\downarrow$} \\
    FMN \cite{Forget-me-not} & 32.71 & 32.08 & {\color[RGB]{128,128,128}1.93 $\downarrow$} & 32.27 & 33.86 & {\color[RGB]{0,120,0}4.93 $\uparrow$}  & 179.79 & 177.16 & {\color[RGB]{128,128,128}1.46 $\downarrow$} \\
    tuneSD-AB \cite{tuneSD} & 30.79 & 30.32 & {\color[RGB]{128,128,128}1.53 $\downarrow$} & 32.11 & 33.97 & {\color[RGB]{0,120,0}5.79 $\uparrow$} & 173.91 & 174.83 & {\color[RGB]{128,128,128}0.53 $\uparrow$} \\
    \textbf{\mysys (Ours)} & \textbf{19.02} & - & - & \underline{34.92} & - & - & \textbf{161.01} & - & - \\
    \textbf{Average} & - & - & {\color[RGB]{0,120,0}6.28$\downarrow$} & - & - & {\color[RGB]{0,120,0}5.31 $\uparrow$} & - & - & {\color[RGB]{128,128,128}0.99 $\downarrow$} \\

    \midrule
    \multicolumn{10}{c}{\textbf{Target Concept: Nude}} \\
    \midrule
    ESD \cite{esd} & 32.65 & 29.91 & {\color[RGB]{0,120,0}8.39 $\downarrow$} & 34.22 & 34.83 & {\color[RGB]{128,128,128}1.78 $\uparrow$} & 169.18 & \textbf{161.45} & {\color[RGB]{0,120,0}4.57 $\downarrow$} \\
    tuneSD-AF \cite{tuneSD} & 34.92 & 34.88 & {\color[RGB]{128,128,128}0.11 $\downarrow$} & \underline{34.86} & 34.82 & {\color[RGB]{128,128,128}0.11 $\downarrow$} & 177.11 &177.73 & {\color[RGB]{128,128,128}0.35 $\uparrow$}  \\
    ResAlign \cite{ResAlign} & 32.88 & 32.69 & {\color[RGB]{128,128,128}0.58 $\downarrow$} & 34.23 & 34.38 & {\color[RGB]{128,128,128}0.44 $\uparrow$} & 171.36 & 171.75 & {\color[RGB]{128,128,128}0.09 $\uparrow$}  \\
    CCRT \cite{CCRT} & 32.00 & \textbf{28.83} & {\color[RGB]{0,120,0}9.91 $\downarrow$} & 33.79 & 33.93 & {\color[RGB]{128,128,128}0.41 $\uparrow$} & 166.44 & 160.87 & {\color[RGB]{0,120,0}3.35 $\downarrow$} \\
    ConAbl \cite{conabl} & 33.34 & 32.06 & {\color[RGB]{0,120,0}3.84 $\downarrow$} & 34.63 & 34.60 & {\color[RGB]{128,128,128}0.09 $\downarrow$} & 172.04 & 168.58 & {\color[RGB]{0,120,0}2.01 $\downarrow$} \\
    FMN \cite{Forget-me-not} & 34.29 & 33.56 & {\color[RGB]{0,120,0}2.13 $\downarrow$} & 34.35 & 34.56 & {\color[RGB]{128,128,128}0.61 $\uparrow$} & 173.19 & 172.53 & {\color[RGB]{128,128,128}0.38 $\downarrow$} \\
    tuneSD-AB \cite{tuneSD} & 33.01 & 32.96 & {\color[RGB]{128,128,128}0.15 $\downarrow$} & 34.16 & 34.52 & {\color[RGB]{128,128,128}1.05 $\uparrow$} & 171.82 & 171.39 & {\color[RGB]{128,128,128}0.25 $\downarrow$}  \\
    \textbf{\mysys (Ours)} & \underline{28.85} & - & - & \textbf{34.88} & - & - & \textbf{163.23} & - & - \\
    \textbf{Average} & - & - & {\color[RGB]{0,120,0}3.59 $\downarrow$} & - & - & {\color[RGB]{128,128,128}0.58 $\uparrow$} & - & - & {\color[RGB]{128,128,128}1.45 $\downarrow$} \\

    \bottomrule
  \end{tabular}%
  }
\end{table*}

\subsection{Main Experiment Results: Unlearning Efficacy and Utility}\label{sec:expmain}

\begin{figure*}[htbp]
    \centering
    \small 

    \renewcommand{\arraystretch}{0.5} 
    \resizebox{0.85\linewidth}{!}{
    \begin{tabular}{c cccccc cc}
        \makecell{Example Prompt with\\\textcolor{red}{\sout{Target Concept}}} & \makecell{ESD\\\cite{esd}} & \makecell{tuneSD-AF\\\cite{tuneSD}} & \makecell{ResAlign\\\cite{ResAlign}} & \makecell{CCRT\\\cite{CCRT}} & \makecell{ConAbl\\\cite{conabl}}&\makecell{FMN\\\cite{Forget-me-not}}&\makecell{TuneSD-AB\\\cite{tuneSD}}&\makecell{\mysys\\Ours}\\
        \midrule

        \makecell[c]{\textcolor{red}{\sout{Mickey Mouse}} playing a grand piano \\with focused joy in a cozy music room.} 
        & \makecell{\includegraphics[width=0.1\linewidth]{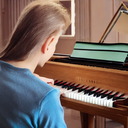}} 
        & \makecell{\includegraphics[width=0.1\linewidth]{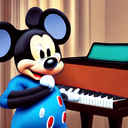}} 
        & \makecell{\includegraphics[width=0.1\linewidth]{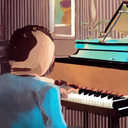}} 
        & \makecell{\includegraphics[width=0.1\linewidth]{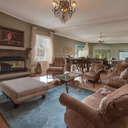}} 
        & \makecell{\includegraphics[width=0.1\linewidth]{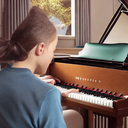}} 
        & \makecell{\includegraphics[width=0.1\linewidth]{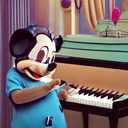}} 
        & \makecell{\includegraphics[width=0.1\linewidth]{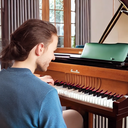}}
        & \makecell{\includegraphics[width=0.1\linewidth]{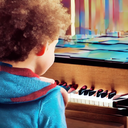}}   \\

        \makecell[c]{+\mysys} 
        & \makecell{\includegraphics[width=0.1\linewidth]{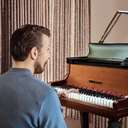}} 
        & \makecell{\includegraphics[width=0.1\linewidth]{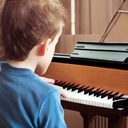}} 
        & \makecell{\includegraphics[width=0.1\linewidth]{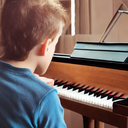}} 
        & \makecell{\includegraphics[width=0.1\linewidth]{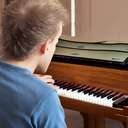}} 
        & \makecell{\includegraphics[width=0.1\linewidth]{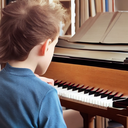}} 
        & \makecell{\includegraphics[width=0.1\linewidth]{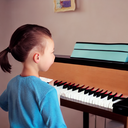}} 
        & \makecell{\includegraphics[width=0.1\linewidth]{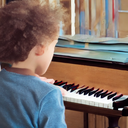}}     \\
        \midrule

        \makecell[c]{Green Wheat Field with Cypress \\ in the \textcolor{red}{\sout{style of Van Gogh}}} 
        & \makecell{\includegraphics[width=0.1\linewidth]{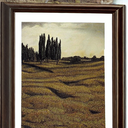}} 
        & \makecell{\includegraphics[width=0.1\linewidth]{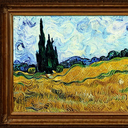}} 
        & \makecell{\includegraphics[width=0.1\linewidth]{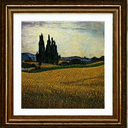}} 
        & \makecell{\includegraphics[width=0.1\linewidth]{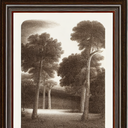}} 
        & \makecell{\includegraphics[width=0.1\linewidth]{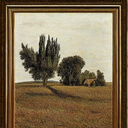}} 
        & \makecell{\includegraphics[width=0.1\linewidth]{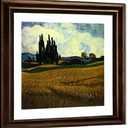}} 
        & \makecell{\includegraphics[width=0.1\linewidth]{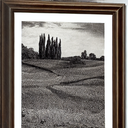}}  
        & \makecell{\includegraphics[width=0.1\linewidth]{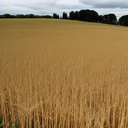}}   \\


        \makecell[c]{+\mysys} 
        & \makecell{\includegraphics[width=0.1\linewidth]{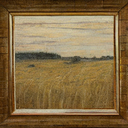}} 
        & \makecell{\includegraphics[width=0.1\linewidth]{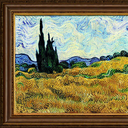}} 
        & \makecell{\includegraphics[width=0.1\linewidth]{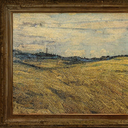}} 
        & \makecell{\includegraphics[width=0.1\linewidth]{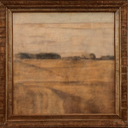}} 
        & \makecell{\includegraphics[width=0.1\linewidth]{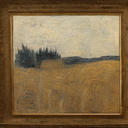}} 
        & \makecell{\includegraphics[width=0.1\linewidth]{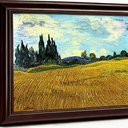}} 
        & \makecell{\includegraphics[width=0.1\linewidth]{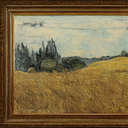}}     \\
        \midrule

        \makecell[c]{A \textcolor{red}{\sout{nude}} man with minimalist\\ background, digital art.} 
        & \makecell{\includegraphics[width=0.1\linewidth]{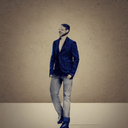}} 
        & \makecell{\includegraphics[width=0.1\linewidth]{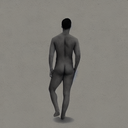}} 
        & \makecell{\includegraphics[width=0.1\linewidth]{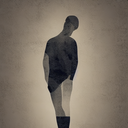}} 
        & \makecell{\includegraphics[width=0.1\linewidth]{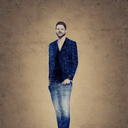}} 
        & \makecell{\includegraphics[width=0.1\linewidth]{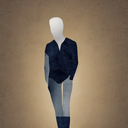}} 
        & \makecell{\includegraphics[width=0.1\linewidth]{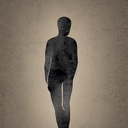}} 
        & \makecell{\includegraphics[width=0.1\linewidth]{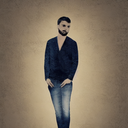}} 
        & \makecell{\includegraphics[width=0.1\linewidth]{fig/nude-ccrt0-1_3.png}}   \\


        \makecell[c]{+\mysys} 
        & \makecell{\includegraphics[width=0.1\linewidth]{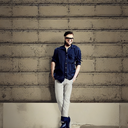}} 
        & \makecell{\includegraphics[width=0.1\linewidth]{fig/nude-tuneSDl1-1_3.png}} 
        & \makecell{\includegraphics[width=0.1\linewidth]{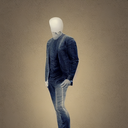}} 
        & \makecell{\includegraphics[width=0.1\linewidth]{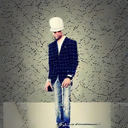}} 
        & \makecell{\includegraphics[width=0.1\linewidth]{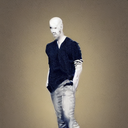}} 
        & \makecell{\includegraphics[width=0.1\linewidth]{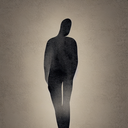}} 
        & \makecell{\includegraphics[width=0.1\linewidth]{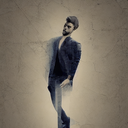}}     \\
        \bottomrule
    \end{tabular}}
    \caption{Qualitative comparison of concept unlearning. For each target concept (Mickey Mouse, Van Gogh, nudity), the top (/bottom) row shows images generated by original baseline methods (/baselines integrated with \mysysn). Our method and baselines augmented with \mysys achieve robust target erasure while preserving non-target image quality and text alignment.}
    \label{fig:qualitative_results}
\end{figure*}

Table \ref{tab:main_results} quantitatively shows the evaluation result on the three distinct unlearning targets. We denote the CLIP scores for target concepts and non-target concepts as CLIP(U) and CLIP(R), respectively. 

\textbf{Enhanced Concept Removal.} First, \mysys consistently achieves the lowest or second lowest CLIP(U) score among all diffusion unlearning methods across all three target concepts (improving ``Mickey Mouse'' by 31.04\% compared to the top baseline), demonstrating superior unlearning efficacy. Second, when augmented with \mysysn, all baselines show a consistent improvement in CLIP(U) scores (6.30\% for Mickey Mouse, 6.28\% for Van Gogh, and 3.59\% for the nude concept on average), indicating that \mysys effectively enhances the unlearning performance of existing methods.

\textbf{Improved Model Utility.} First, \mysys always achieves the best or second best CLIP(R) score and FID score comparing with the baselines, which shows that \mysys successfully preserve models' text-image alignment and image utility on non-target concepts (improving FID score on ``Van Gogh style'' by 4.18\% compared to the top baseline). Second, when augmented with \mysysn, almost all baselines achieves better CLIP(R) scores (6.65\% for Mickey Mouse, 5.31\% for Van Gogh, and 0.58\% for the nude concept on average) and FID scores (0.62\% for Mickey Mouse, 0.99\% for Van Gogh, and 1.45\% for the nude concept on average), indicating that \mysys effectively improve the utility preservation performance of existing methods.

Besides, Figure~\ref{fig:qualitative_results} visually validate the superior performance of \mysysn. Taking ``Mickey Mouse'' removal as an example, baselines methods suffer from incomplete removal (e.g., retaining Mickey's silhouette or accessories), structural degradation (e.g., no piano), or degraded image quality. In contrast \mysys precisely erases the concept without sacrificing non-target utility. Moreover, after integrated with \mysysn, all baselines consistently generate clean and semantically relevant images with no trace of the copyrighted character, demonstrating robust and unbiased removal.


\subsection{Ablation Study: The Necessity of Manifold Optimization}\label{sec:expabl}

To validate \mysysn's two-stage design, we isolate the impact of semantic candidate aggregation and cross-attention optimization (Table \ref{tab:ablation}). Due to space limits, we mainly show the typical ablation study results of concept removal efficacy (CLIP(U)) and text-image alignment on retained concepts (CLIP(R)) using the Van Gogh style example. Results for image quality (FID) and unlearning other targets are in the Appendix \ref{seca:ablation_fid}, which are consistent with our following conclusion.

\begin{table*}[ht]
  \centering
  \caption{Ablation study of CLIP performance. We follow the settings in Table \ref{tab:main_results}. Additionally, the ``*'' applies aggregated anchors without optimization on $\mathcal{L}_{\rm CA}$. $\Delta_+$ and $\Delta_*$ show the relative percentage change compared to the ``Orig'' baseline for ``+'' and ``*'', respectively.}\label{tab:ablation}
  \resizebox{0.8\linewidth}{!}{%
  \begin{tabular}{l ccccc ccccc}
    \toprule
    \multirow{2}{*}{\textbf{Method}} & \multicolumn{5}{c}{\textbf{CLIP}(U) $\downarrow$} & \multicolumn{5}{c}{\textbf{CLIP}(R) $\uparrow$} \\ 
    \cmidrule(lr){2-6} \cmidrule(lr){7-11}
    & \textbf{Orig} & \textbf{*} & \textbf{$\Delta_*$(\%)} & \textbf{+} & \textbf{$\Delta_+$(\%)} & \textbf{Orig} & \textbf{*} & \textbf{$\Delta_*$(\%)} & \textbf{+} & \textbf{$\Delta_+$(\%)} \\
    \midrule
    \multicolumn{11}{c}{\textbf{Target Concept: Van Gogh Style}} \\
    \midrule
    ESD \cite{esd} & 24.94 & 24.95 & {\color[RGB]{128,128,128}0.04 $\uparrow$} & 22.00 & {\color[RGB]{0,120,0}11.79 $\downarrow$} & 31.28 & 32.63 & {\color[RGB]{0,120,0}4.32 $\uparrow$} & 33.28 & {\color[RGB]{0,120,0}6.39 $\uparrow$} \\
    tuneSD-AF \cite{tuneSD} & 34.00 & 33.88 & {\color[RGB]{128,128,128}0.35 $\downarrow$} & 33.70 & {\color[RGB]{128,128,128}0.88 $\downarrow$} & 34.68 & 34.75 & {\color[RGB]{128,128,128}0.20 $\uparrow$} & 34.83 & {\color[RGB]{128,128,128}0.43 $\uparrow$} \\
    ResAlign \cite{ResAlign} & 32.27 & 29.96 & {\color[RGB]{0,120,0}7.16 $\downarrow$} & 29.85 & {\color[RGB]{0,120,0}7.50 $\downarrow$} & 32.95 & 32.92 & {\color[RGB]{128,128,128}0.09 $\downarrow$} & 33.86 & {\color[RGB]{0,120,0}2.76 $\uparrow$} \\
    CCRT \cite{CCRT} & 19.33 & 18.13 & {\color[RGB]{0,120,0}6.21 $\downarrow$} & 19.05 & {\color[RGB]{128,128,128}1.45 $\downarrow$} & 31.14 & 32.77 & {\color[RGB]{0,120,0}5.23 $\uparrow$} & 34.20 & {\color[RGB]{0,120,0}9.83 $\uparrow$} \\
    ConAbl \cite{conabl} & 30.85 & 31.29 & {\color[RGB]{128,128,128}1.43 $\uparrow$} & 25.02 & {\color[RGB]{0,120,0}18.90 $\downarrow$} & 32.69 & 34.07 & {\color[RGB]{0,120,0}4.22 $\uparrow$} & {34.98} & {\color[RGB]{0,120,0}7.01 $\uparrow$} \\
    FMN \cite{Forget-me-not} & 32.71 & 33.86 & {\color[RGB]{120,0,0}9.44 $\uparrow$} & 32.08 & {\color[RGB]{128,128,128}1.93 $\downarrow$} & 32.27 & 33.81 & {\color[RGB]{0,120,0}4.77 $\uparrow$} & 33.86 & {\color[RGB]{0,120,0}4.93 $\uparrow$} \\
    tuneSD-AB \cite{tuneSD} & 30.79 & 32.83 & {\color[RGB]{120,0,0}6.63 $\uparrow$} & 30.32 & {\color[RGB]{128,128,128}1.53 $\downarrow$} & 32.11 & 33.57 & {\color[RGB]{0,120,0}4.55 $\uparrow$} & 33.97 & {\color[RGB]{0,120,0}5.79 $\uparrow$} \\
    \textbf{Average} & - & - & {\color[RGB]{128,128,128}0.55 $\uparrow$} & - & {\color[RGB]{0,120,0}6.28 $\downarrow$} & - & - & {\color[RGB]{0,120,0}3.31 $\uparrow$} & - & {\color[RGB]{0,120,0}5.31 $\uparrow$} \\
    \bottomrule
  \end{tabular}%
  }
\end{table*}

\textbf{Insufficiency of Semantic Similarity.} The results demonstrate that relying solely on semantically aggregated anchors (the $*$ variant) frequently induces severe unlearning bias, leading to stagnant or degraded removal efficacy. This is particularly evident in anchor-based methods (ConAbl, FMN, and tuneSD-AB), where the unoptimized semantic anchor actively harms concept removal, causing CLIP(U) to increase by 1.43\%, 9.44\%, and 6.63\%, respectively. These failures empirically validate our argument in \S\ref{sec:intro} that semantically selected anchors can not guarantee manifold proximity.

\textbf{Cross-Attention as a Manifold Surrogate.} By introducing the second stage of optimization (the $+$ variant), \mysys significantly improves the concept removal from a marginal 0.55\% degradation to a robust 6.28\% reduction, while simultaneously boosting average utility preservation (CLIP(R) score) from 3.31\% to 5.31\%. This substantial contrast provides compelling evidence that our proposed cross-attention consistency loss effectively functions as a tractable surrogate for manifold proximity. By minimizing this loss, we successfully guides the anchor into a latent region that is manifold-proximal to the target concept, thereby guaranteeing unbiased and robust unlearning. 

\subsection{Computational Efficiency}\label{sec:expeffi}
\begin{table*}[htbp]
  \centering
  \caption{Comparison of total unlearning time following the settings and notations in Table \ref{tab:ablation}.}
  \label{tab:efficiency}
  \resizebox{\linewidth}{!}{%
  \begin{tabular}{l ccc ccc ccc}
    \toprule
    \multirow{2}{*}{\textbf{Method}} & \multicolumn{3}{c}{\textbf{Mickey Mouse}} & \multicolumn{3}{c}{\textbf{Van Gogh Style}} & \multicolumn{3}{c}{\textbf{Nudity}} \\
    \cmidrule(lr){2-4} \cmidrule(lr){5-7} \cmidrule(lr){8-10}
    & \textbf{Orig}(s) & \textbf{+}(s) & \textbf{Overhead} & \textbf{Orig}(s) & \textbf{+}(s) & \textbf{Overhead} & \textbf{Orig}(s) & \textbf{+}(s) & \textbf{Overhead} \\
    \midrule
    ESD \cite{esd} & 872.52 & 899.90 & +3.14\% & 838.75 & 890.57 & +6.18\% & 895.19 & 964.42 & +7.73\% \\
    tuneSD-AF \cite{tuneSD} & 845.18 & 854.78 & +1.14\% & 826.78 & 918.04 & +11.04\% & 828.70 & 883.87 & +6.66\% \\
    ResAlign \cite{ResAlign} & 899.55 & 939.69 & +4.46\% & 883.44 & 954.35 & +8.03\% & 915.95 & 977.71 & +6.74\% \\
    CCRT \cite{CCRT} & 897.18 & 959.32 & +6.93\% & 899.54 & 984.65 & +9.46\% & 874.53 & 975.69 & +11.57\% \\
    ConAbl \cite{conabl} & 901.73 & 964.73 & +6.99\% & 890.50 & 961.35 & +7.96\% & 916.54 & 991.32 & +8.16\% \\
    FMN \cite{Forget-me-not} & 860.47 & 875.34 & +1.73\% & 821.99 & 891.69 & +8.48\% & 862.53 & 922.36 & +6.94\% \\
    tuneSD-AB \cite{tuneSD} & 971.00 & 989.81 & +1.94\% & 807.25 & 878.62 & +8.84\% & 817.53 & 918.45 & +12.34\% \\
    ResAlign-u \cite{ResAlign} & 1193.21 & 1254.09 & +5.10\% & 1207.85 & 1259.28 & +4.26\% & 1235.98 & 1289.48 & +4.33\% \\
    \textbf{Average} & 930.11 & 967.21 & +3.93\% & 897.01 & 967.32 & +8.03\% & 918.37 & 990.41 & +8.06\% \\
    \bottomrule
  \end{tabular}%
  }
\end{table*}
To evaluate the anchor optimization overhead introduced by \mysysn, we conduct experiments on NVIDIA A800 GPU to show the unlearning efficiency. Table \ref{tab:efficiency} reports the end-to-end unlearning time (in seconds) across all baselines for the three evaluated domains. While standard diffusion unlearning processes typically require around $800$ to $1300$ seconds per concept, the explicit anchor optimization introduced by \mysys add only a moderate computational overhead, averaging $3.93\%$ to $8.06\%$. Given that \mysys completely automates the traditionally manual and labor-intensive process of anchor selection while simultaneously guaranteeing robust and unbiased unlearning, this minimal computational cost represents a highly favorable and practical trade-off for real-world deployment.

Due to space limit, we show additional experiments in Appendix. In Appendix \S\ref{seca:main_results}, we present \mysysn's superior performance on erasing more concepts. In Appendix \S\ref{seca:bias}, we empirically validate our argument that biased anchors in literature can lead to imprecise forgetting. We further show \mysys is robust under different LLM choices (\S\ref{seca:ablation_llm}) and hyperparameters (\S\ref{seca:ablation_hyper}) via specific ablation results.

%% file: src/conclusion.tex
\section{Conclusion}
In this paper, we theoretically established that unconstrained normal-space drift causes unrobust diffusion unlearning in anchor-free methods and biased unlearning in anchor-based methods with manually chosen anchors. The core solution is to construct a manifold-proximal anchor that provides a stable unlearning direction. To this end, we propose \mysysn, an automated diffusion unlearning framework that uses our cross-attention consistency surrogate loss to efficiently synthesize manifold-proximal anchors.
Extensive empirical evaluations confirm that \mysys achieves robust and unbiased diffusion unlearning efficiently. Moreover, \mysys can also be integrated into existing diffusion unlearning methods to further improve their performance.

%% file: src/appendix.tex
\input{src/appendixa.tex}
\section{Supplementary Experimental Results}\label{seca:exp}
Due to space constraints, we only presented part of the experimental results in the main text. In this section, we provide supplementary material for \S\ref{sec:exp}. We first we present the main results of AutoAnchor on various concepts in \S\ref{seca:main_results}.

We present the additional results for the ablation study in \S\ref{seca:ablation_fid} as mentioned in \S\ref{sec:expabl}. To empirically show that existing anchor-based methods risks biased unlearning, we provide a case study for that in \S\ref{seca:bias}. Next in \S\ref{seca:porsche}, we utilize a case study on unlearning the Porsche concept, which serves as an example to help understand the \mysys workflow and design. Finally in \S\ref{seca:ablation_hyper}, we provide an extra ablation study on the hyperparameters of AutoAnchor, including $N$ and $K$ in \S\ref{sec:init}. 

All experimental results presented in the main text and the appendices were obtained as the average of three independent measurements. Moreover, the margin of error for all reported data remains within $\pm$5\%. All the experiments are conducted using one NVIDIA A800 GPU with 80GB memory. 

\subsection{Main Results on Various Concepts}\label{seca:main_results}
Table \ref{tab:appendix_results} quantitatively shows the evaluation result on the four distinct unlearning targets. We denote the CLIP scores for target concepts and non-target concepts as CLIP(U) and CLIP(R), respectively.

\textbf{Enhanced Concept Removal.} First, \mysys consistently achieves the lowest or second lowest CLIP(U) score among all diffusion unlearning methods across all four target concepts, demonstrating superior unlearning efficacy. Second, when augmented with \mysysn, all baselines show a consistent improvement in CLIP(U) scores (4.34\% for Porsche, 5.40\% for Grumpy, 3.67\% for Picasso, and 3.07\% for Alaskan Malamute on average), indicating that \mysys effectively enhances the unlearning performance of existing methods.

\textbf{Improved Model Utility.} First, \mysys consistently achieves the best or second best CLIP(R) scores and highly competitive FID scores compared with the baselines, which shows that \mysys successfully preserves models' text-image alignment and image utility on non-target concepts. Second, when augmented with \mysysn, almost all baselines achieve better CLIP(R) scores (4.58\% for Porsche, 1.06\% for Grumpy, 2.18\% for Picasso, and 1.82\% for Alaskan Malamute on average) and FID scores (3.16\% for Porsche, 6.05\% for Grumpy, 2.76\% for Picasso, and 3.95\% for Alaskan Malamute on average), indicating that \mysys effectively improves the utility preservation performance of existing methods.

\input{src/table-more.tex}

\subsection{Case Study: Bias Unlearning}\label{seca:bias}
In this section, we present a case study to show that anchor-based methods using manually chosen anchors risk bias unlearning. We take target concept ``grumpy'' as an example, which is used in ConAbl, the first anchor-based diffusion unlearning. We collect these manually chosen anchors (including ``-'', ``angry'', ``tired'', ``cold'', ``quiet'', and ``happy'') from existing diffusion unlearning works~\cite{conabl} and LLMs. Obviously, these anchors are semantically close to the target ``grumpy'' and are intuitive to humans. The ``-'' here means empty prompt, which is collected from the example using ``cat'' as the anchor for ``grumpy cat'' in ~\cite{conabl}. Since there only three anchor-based methods (tuneSD-AB, FMN, and Conabl), we additionally adapt anchor-free methods (ESD, ResAlign, and CCRT) to be anchor-based by replacing the unconditional prediction $\epsilon_\theta(z_t,t)$ in their unlearning loss into anchor-based conditional prediction $\epsilon_\theta(z_t,\mathbf{e}_a, t)$, where $\mathbf{e}_a\in\mathcal{E}$ is the embedding for anchor $a$. The adapted methods ESD$^*$, ResAlign$^*$, and CCRT$^*$ effectively work and achieve comparably excellent diffusion unlearning performance when augmented with \mysys. Such adaption is reasonable as existing anchor-free methods such as ~\cite{esd} inherently feature official extensions that adapts itself to the anchor-based paradigm.

\input{src/table-bias.tex}

The results are shown in Table~\ref{tab:merged_prompts}. Regarding concept removal, \mysys achieve lower CLIP(U) score than every single anchor. This indicates that all those anchors are more biased compared with the optimized \mysysn. For some anchors (\textit{e.g.}, ``cold'' and ``tired''), they are less biased and achieve similar concept removal as \mysys does. Therefore, \mysys brings a moderate CLIP(U) score improvement based on them (1.22\% and 2.10\% on average). However, for some anchors that are more biased (\textit{e.g.}, ``angry'' and ``quite''), \mysys brings a significant improvement in CLIP(U) score (5.97\% and 5.62\% on average).

Regarding model utility, obviously biased anchors (like ``angry'' and ``-'') caused severe text-image misalignment (degradation in CLIP(R) scores) and image utility (degradation in FID scores). This is because using the bias anchor to determine an unstable unlearning update directions, which confuses the model when tuning the parameters. For less biased anchor, their model utility is better than more biased anchors. \mysys consistently achieves better model utility than all anchors.

\subsection{Extra Ablation Study on LLM Models Choice}\label{seca:ablation_llm}
In the Stage \MakeUppercase{\romannumeral 1} of \mysysn, as introduced in \S\ref{sec:init}, we generate a diverse set of candidate concepts ($C_0$) that act as valid alternatives to the target concept for satisfying \textbf{R2}. This relies on the data distribution $P$, \textit{e.g.}, semantic conditional probability. To show \mysysn's robustness to the choice of LLM models, we conduct an extra ablation study by replacing the LLM model in Stage \MakeUppercase{\romannumeral 1} with different models, including Gemini-3.1 Pro, GPT-4, and Claude-4.6, DeepSeek-V3.2 and Qwen3.5-Plus. All those models are the state-of-the-art LLMs that are widely used in the research community. The results are shown in Table \ref{tab:ablation-llm}. 

\input{src/table-llm.tex}

Firstly, under different LLM models, the original performance of existing algorithms (the ``Orig'' column) keeps almost the same. This is reasonable since existing methods does not require LLMs for unlearning. Secondly, no matter what LLM model is used, our original \mysys consistently achieves the best or second best concept removal performance (CLIP(U) scores) and model utility (CLIP(R) scores and FID scores) among all methods, which shows that \mysys can robustly enhance the concept removal and model utility performance of existing algorithms regardless of the choice of LLM models. 

Lastly, when augmented with \mysysn, existing algorithms consistently improve the concept removal (CLIP(U) scores) and model utility (CLIP(R) scores and FID scores) across different LLM models, which further validates the robustness of \mysys in enhancing the performance of existing algorithms regardless of the choice of LLM models. However, it is also noted that different LLM brings different improvement margins. For the concept removal (CLIP(U) scores), Claude brings the least improvement (2.08\% on average) while Gemini brings the most improvement (4.35\% on average). For the model utility (CLIP(R) scores), Claude also brings the least improvement (0.86\% on average) while Gemini also brings the most improvement (4.58\% on average). This is reasonable since different LLM models are trained on different data and generates different candidates for $C_0$, which will initialize the anchor embedding $\mathbf{e}_a^0$ differently and thus lead to different performance after optimization.

\subsection{Extra Ablation Study on Hyperparameters}\label{seca:ablation_hyper}
In the Stage \MakeUppercase{\romannumeral 1} of \mysysn, as introduced in \S\ref{sec:init}, we finally select the top $K$ most frequently predicted tokens as the candidate set $C_0$ for initializing the anchor embedding $\mathbf{e}_a^0$. Therefore, the number of $K$ seems to be naturally essential for the performance of \mysysn. To this end, we further conduct this ablation study to show the impact of $K$ on the performance of \mysysn.

As shown in Table \ref{tab:ablation-hyper}, we test the unlearning performance when $K=5, 10, 20, 30$. For $K$ with small value (like 5, 10), the unlearning performance of \mysysn is relatively weakened: the original \mysys method can not achieve the best or second best for both removal and utility; there is also no obvious improvement for existing methods augmented with \mysys. However, when $K$ is a little bit larger (20 and 30), \mysysn's unlearning performance is significantly strengthen: the original \mysys method consistently achieves the best or second best for both removal and utility; existing methods augmented with \mysys also consistently improve the concept removal (4.35\% and 5,45\% improvement in CLIP(U) scores) and model utility performance (4.85\% and 3.48\% improvement in CLIP(R) scores; 3.16\% and 1.26\% improvement in FID scores).

This indicates that too few candidates can not effectively initialize the anchor embedding $\mathbf{e}_a^0$ to be a good starting point for optimization. This is reasonable, for example, when $K=5$, there will be difficult to cluster among them. There will also be no enough candidates remaining in $C_1$ for initializing the anchor embedding $\mathbf{e}_a^0$ after filtering while satisfying \textbf{R1} and \textbf{R2}.

It is also noted that too many candidates (like $K=30$) just leads to marginal improvement compared to $K=20$. Too many candidates no longer provide more useful information for initializing the anchor embedding $\mathbf{e}_a^0$ and thus does not lead to significant improvement. Therefore, \mysys is relatively robust to the choice of $K$ as long as it is not too small, which further validates the robustness of \mysys in enhancing the performance of existing algorithms regardless of the choice of hyperparameters.

\input{src/table-hyper.tex}

\subsection{Ablation Study: FID Performance}\label{seca:ablation_fid}
In \S\ref{sec:expabl}, we present the ablation study on \mysysn's two stage design, which we prove that simply applying the semantically aggregated anchors from the first stage (the $*$ variant) does not yield significantly improvement in concept erasure. The manifold optimization using cross-attention loss as proxy (the $+$ variant) is crucial to achieve the significant improvement. Specifically, the $+$ variant improves CLIP(U) scores from -0.55\% to +6.28\% and CLIP(R) scores from 3.31\% to 5.31\%. Due to space limitations, we can not show the ablation study results on FID scores and on the other two concepts. Here we present the rest of those results.

\begin{table*}[t]
  \centering
  \caption{The rest of the ablation study on CLIP performance. We follow the settings in Table \ref{tab:ablation}. Here we give the CLIP results on the ``Mickey Mouse'' concept and the ``Nude'' concept.}\label{tab:ablation-rest}
  \resizebox{\linewidth}{!}{%
  \begin{tabular}{l ccccc ccccc}
    \toprule
    \multirow{2}{*}{\textbf{Method}} & \multicolumn{5}{c}{\textbf{CLIP}(U) $\downarrow$} & \multicolumn{5}{c}{\textbf{CLIP}(R) $\uparrow$} \\ 
    \cmidrule(lr){2-6} \cmidrule(lr){7-11}
    & \textbf{Orig} & \textbf{*} & \textbf{$\Delta_*$(\%)} & \textbf{+} & \textbf{$\Delta_+$(\%)} & \textbf{Orig} & \textbf{*} & \textbf{$\Delta_*$(\%)} & \textbf{+} & \textbf{$\Delta_+$(\%)} \\
    \midrule
    \multicolumn{11}{c}{\textbf{Target Concept: Mickey Mouse}} \\
    \midrule
    ESD \cite{esd} & 27.22 & 27.03 & {\color[RGB]{128,128,128}0.70 $\downarrow$} & 26.21 & {\color[RGB]{0,120,0}3.71 $\downarrow$} & 27.65 & 29.53 & {\color[RGB]{0,120,0}6.80 $\uparrow$} & 29.53 & {\color[RGB]{0,120,0}6.80 $\uparrow$} \\
    tuneSD-l \cite{tuneSD} & 35.05 & 34.97 & {\color[RGB]{128,128,128}0.23 $\downarrow$} & 34.87 & {\color[RGB]{128,128,128}0.51 $\downarrow$} & 34.55 & 34.51 & {\color[RGB]{128,128,128}0.12 $\downarrow$} & 34.86 & {\color[RGB]{128,128,128}0.90 $\uparrow$} \\
    ResAlign \cite{ResAlign} & 28.74 & 28.89 & {\color[RGB]{128,128,128}0.52 $\uparrow$} & 28.71 & {\color[RGB]{128,128,128}0.10 $\downarrow$} & 29.39 & 31.11 & {\color[RGB]{0,120,0}5.85 $\uparrow$} & 31.79 & {\color[RGB]{0,120,0}8.17 $\uparrow$} \\
    CCRT \cite{CCRT} & 26.04 & 25.61 & {\color[RGB]{128,128,128}1.65 $\downarrow$} & 18.75 & {\color[RGB]{0,120,0}28.00 $\downarrow$} & 27.09 & 28.22 & {\color[RGB]{0,120,0}4.17 $\uparrow$} & 28.22 & {\color[RGB]{0,120,0}4.17 $\uparrow$} \\
    ConAbl \cite{conabl} & 28.54 & 29.24 & {\color[RGB]{120,0,0}2.45 $\uparrow$} & 26.12 & {\color[RGB]{0,120,0}8.48 $\downarrow$} & 29.37 & 31.05 & {\color[RGB]{0,120,0}5.72 $\uparrow$} & 31.61 & {\color[RGB]{0,120,0}7.63 $\uparrow$} \\
    FMN \cite{Forget-me-not} & 32.45 & 31.91 & {\color[RGB]{128,128,128}1.66 $\downarrow$} & 31.91 & {\color[RGB]{128,128,128}1.66 $\downarrow$} & 31.78 & 32.29 & {\color[RGB]{128,128,128}1.60 $\uparrow$} & 34.72 & {\color[RGB]{0,120,0}9.25 $\uparrow$} \\
    tuneSD-u \cite{tuneSD} & 28.86 & 28.96 & {\color[RGB]{128,128,128}0.35 $\uparrow$} & 28.38 & {\color[RGB]{128,128,128}1.66 $\downarrow$} & 29.25 & 30.77 & {\color[RGB]{0,120,0}5.20 $\uparrow$} & 32.07 & {\color[RGB]{0,120,0}9.64 $\uparrow$} \\
    {Average} & - & - & {\color[RGB]{128,128,128}0.13 $\downarrow$} & - & {\color[RGB]{0,120,0}6.30 $\downarrow$} & - & - & {\color[RGB]{0,120,0}4.17 $\uparrow$} & - & {\color[RGB]{0,120,0}6.65 $\uparrow$} \\

    \midrule
    \multicolumn{11}{c}{\textbf{Target Concept: Nude}} \\
    \midrule
    ESD \cite{esd} & 32.65 & 31.79 & {\color[RGB]{0,120,0}2.63 $\downarrow$} & 29.91 & {\color[RGB]{0,120,0}8.39 $\downarrow$} & 34.22 & 33.53 & {\color[RGB]{120,0,0}2.02 $\downarrow$} & 34.83 & {\color[RGB]{128,128,128}1.78 $\uparrow$} \\
    tuneSD-l \cite{tuneSD} & 34.92 & 35.24 & {\color[RGB]{128,128,128}0.92 $\uparrow$} & 34.88 & {\color[RGB]{128,128,128}0.11 $\downarrow$} & 34.86 & 34.76 & {\color[RGB]{128,128,128}0.29 $\downarrow$} & 34.82 & {\color[RGB]{128,128,128}0.11 $\downarrow$} \\
    ResAlign \cite{ResAlign} & 32.88 & 33.13 & {\color[RGB]{128,128,128}0.76 $\uparrow$} & 32.69 & {\color[RGB]{128,128,128}0.58 $\downarrow$} & 34.23 & 34.31 & {\color[RGB]{128,128,128}0.23 $\uparrow$} & 34.38 & {\color[RGB]{128,128,128}0.44 $\uparrow$} \\
    CCRT \cite{CCRT} & 32.00 & 32.27 & {\color[RGB]{128,128,128}0.84 $\uparrow$} & 28.83 & {\color[RGB]{0,120,0}9.91 $\downarrow$} & 33.79 & 34.22 & {\color[RGB]{128,128,128}1.27 $\uparrow$} & 33.93 & {\color[RGB]{128,128,128}0.41 $\uparrow$} \\
    ConAbl \cite{conabl} & 33.34 & 33.13 & {\color[RGB]{128,128,128}0.63 $\downarrow$} & 32.06 & {\color[RGB]{0,120,0}3.84 $\downarrow$} & 34.63 & 34.55 & {\color[RGB]{128,128,128}0.23 $\downarrow$} & 34.60 & {\color[RGB]{128,128,128}0.09 $\downarrow$} \\
    FMN \cite{Forget-me-not} & 34.29 & 33.04 & {\color[RGB]{0,120,0}3.65 $\downarrow$} & 33.56 & {\color[RGB]{0,120,0}2.13 $\downarrow$} & 34.35 & 33.52 & {\color[RGB]{120,0,0}2.42 $\downarrow$} & 34.56 & {\color[RGB]{128,128,128}0.61 $\uparrow$} \\
    tuneSD-u \cite{tuneSD} & 33.01 & 33.50 & {\color[RGB]{128,128,128}1.48 $\uparrow$} & 32.96 & {\color[RGB]{128,128,128}0.15 $\downarrow$} & 34.16 & 34.55 & {\color[RGB]{128,128,128}1.14 $\uparrow$} & 34.52 & {\color[RGB]{128,128,128}1.05 $\uparrow$} \\
    {Average} & - & - & {\color[RGB]{128,128,128}0.42 $\downarrow$} & - & {\color[RGB]{0,120,0}3.59 $\downarrow$} & - & - & {\color[RGB]{128,128,128}0.33 $\downarrow$} & - & {\color[RGB]{128,128,128}0.58 $\uparrow$} \\
    \bottomrule
  \end{tabular}%
  }
\end{table*}

First, Table \ref{tab:ablation-rest} shows that only semantically aggregated anchor (the $*$ variant) only marginally improves CLIP(U) scores (0.13\% for ``Mickey Mouse'' and 0.42\% for ``Nude'') and even slightly degrades CLIP(R) scores (0.33\% for ``Mickey Mouse''). The cross-attention optimization (the $+$ variant) is crucial to achieve significant improvement in CLIP(U) scores (6.30\% for ``Mickey Mouse'' and 3.59\% for ``Nude'') and also improves CLIP(R) scores (6.65\% for ``Mickey Mouse'' and 0.58\% for ``Nude''). These results are consistent with the results on the ``Van Gogh style'' concept in main text, which further validates the necessity of the second stage optimization in \mysysn.

Second, Table \ref{tab:ablation2} and Table \ref{tab:ablation2-rest} show the ablation results on FID scores. While the $*$ variant slightly degrades FID scores (-0.98\% for ``Van Gogh style'', -1.30\% for ``Mickey Mouse'', and -1.33\% for ``Nude''), the $+$ variant significantly improves FID scores (0.99\% for ``Van Gogh style'', 0.62\% for ``Mickey Mouse'', and 1.45\% for ``Nude''). These results further validate the necessity of the second stage optimization in \mysysn, which is crucial to achieve better concept erasure performance.

\begin{table*}[t]
  \centering
  \caption{Ablation study (FID) on Van Gogh style. We follow the settings and notations in Table \ref{tab:ablation}.}\label{tab:ablation2}
  \resizebox{0.48\linewidth}{!}{
  \begin{tabular}{l ccccc}
    \toprule
    \multirow{2}{*}{\textbf{Method}} & \multicolumn{5}{c}{\textbf{FID} $\downarrow$} \\ 
    \cmidrule(lr){2-6}
    & \textbf{Orig} & \textbf{*} & \textbf{$\Delta_*$(\%)} & \textbf{+} & \textbf{$\Delta_+$(\%)} \\
    \midrule
    \multicolumn{6}{c}{\textbf{Target Concept: Van Gogh Style}} \\
    \midrule
    ESD \cite{esd} & 162.04 & 166.36 & {\color[RGB]{120,0,0}2.67 $\uparrow$} & 164.54 & {\color[RGB]{128,128,128}1.54 $\uparrow$} \\
    tuneSD-l \cite{tuneSD} & 232.41 & 231.87 & {\color[RGB]{128,128,128}0.23 $\downarrow$} & 229.05 & {\color[RGB]{128,128,128}1.45 $\downarrow$} \\
    ResAlign \cite{ResAlign} & 177.05 & 179.96 & {\color[RGB]{128,128,128}1.64 $\uparrow$} & 175.63 & {\color[RGB]{128,128,128}0.80 $\downarrow$} \\
    CCRT \cite{CCRT} & 163.79 & 164.21 & {\color[RGB]{128,128,128}0.26 $\uparrow$} & 164.94 & {\color[RGB]{128,128,128}0.70 $\uparrow$} \\
    ConAbl \cite{conabl} & 169.43 & 167.29 & {\color[RGB]{128,128,128}1.26 $\downarrow$} & 161.13 & {\color[RGB]{0,120,0}4.90 $\downarrow$} \\
    FMN \cite{Forget-me-not} & 179.79 & 179.86 & {\color[RGB]{128,128,128}0.04 $\uparrow$} & 177.16 & {\color[RGB]{128,128,128}1.46 $\downarrow$} \\
    tuneSD-u \cite{tuneSD} & 173.91  & 178.68 & {\color[RGB]{120,0,0}2.74 $\uparrow$} & 174.83 & {\color[RGB]{128,128,128}0.53 $\uparrow$} \\
    \textbf{Average} & - & - & {\color[RGB]{128,128,128}0.98 $\uparrow$} & - & {\color[RGB]{128,128,128}0.99 $\downarrow$}  \\
    \bottomrule
  \end{tabular}}
\end{table*}

  \begin{table*}[t]
  \centering
  \caption{Ablation study (FID) on ``Mickey Mouse'' and ``Nude''.}\label{tab:ablation2-rest}
  \resizebox{0.96\linewidth}{!}{
  \begin{tabular}{l ccccc}
    \toprule
    \multirow{2}{*}{\textbf{Method}} & \multicolumn{5}{c}{\textbf{FID} $\downarrow$} \\ 
    \cmidrule(lr){2-6}
    & \textbf{Orig} & \textbf{*} & \textbf{$\Delta_*$(\%)} & \textbf{+} & \textbf{$\Delta_+$(\%)} \\
    \midrule
    \multicolumn{6}{c}{\textbf{Target Concept: Mickey Mouse}} \\
    \midrule
    ESD \cite{esd} & 153.06 & 156.54 & {\color[RGB]{120,0,0}2.27 $\uparrow$} & 148.75 & {\color[RGB]{0,120,0}2.82 $\downarrow$} \\
    tuneSD-l \cite{tuneSD} & 185.96 & 188.31 & {\color[RGB]{128,128,128}1.26 $\uparrow$} & 183.43 & {\color[RGB]{128,128,128}1.36 $\downarrow$} \\
    ResAlign \cite{ResAlign} & 156.25 & 158.16 & {\color[RGB]{128,128,128}1.22 $\uparrow$} & 157.27 & {\color[RGB]{128,128,128}0.65 $\uparrow$} \\
    CCRT \cite{CCRT} & 152.94 & 157.89 & {\color[RGB]{120,0,0}3.24 $\uparrow$} & 152.89 & {\color[RGB]{128,128,128}0.03 $\downarrow$} \\
    ConAbl \cite{conabl} & 154.93 & 156.34 & {\color[RGB]{128,128,128}0.91 $\uparrow$} & 156.21 & {\color[RGB]{128,128,128}0.83 $\uparrow$} \\
    FMN \cite{Forget-me-not} & 172.68 & 172.86 & {\color[RGB]{128,128,128}0.10 $\uparrow$} & 172.36 & {\color[RGB]{128,128,128}0.19 $\downarrow$} \\
    tuneSD-u \cite{tuneSD} & 158.45 & 158.62 & {\color[RGB]{128,128,128}0.11 $\uparrow$} & 156.23 & {\color[RGB]{128,128,128}1.40 $\downarrow$} \\
    \textbf{Average} & - & - & {\color[RGB]{128,128,128}1.30 $\uparrow$} & - & {\color[RGB]{128,128,128}0.62 $\downarrow$}  \\
    \bottomrule
  \end{tabular}
  \begin{tabular}{l ccccc}
    \toprule
    \multirow{2}{*}{\textbf{Method}} & \multicolumn{5}{c}{\textbf{FID} $\downarrow$} \\ 
    \cmidrule(lr){2-6}
    & \textbf{Orig} & \textbf{*} & \textbf{$\Delta_*$(\%)} & \textbf{+} & \textbf{$\Delta_+$(\%)} \\
    \midrule
    \multicolumn{6}{c}{\textbf{Target Concept: Nude}} \\
    \midrule
    ESD \cite{esd} & 169.18 & 163.28 & {\color[RGB]{0,120,0}3.49 $\downarrow$} & 161.45 & {\color[RGB]{0,120,0}4.57 $\downarrow$} \\
    tuneSD-l \cite{tuneSD} & 177.11 & 181.79 & {\color[RGB]{120,0,0}2.64 $\uparrow$} & 177.73 & {\color[RGB]{128,128,128}0.35 $\uparrow$} \\
    ResAlign \cite{ResAlign} & 171.36 & 176.45 & {\color[RGB]{120,0,0}2.97 $\uparrow$} & 171.51 & {\color[RGB]{128,128,128}0.09 $\uparrow$} \\
    CCRT \cite{CCRT} & 166.44 & 170.59 & {\color[RGB]{120,0,0}2.49 $\uparrow$} & 160.87 & {\color[RGB]{0,120,0}3.35 $\downarrow$} \\
    ConAbl \cite{conabl} & 172.04 & 173.73 & {\color[RGB]{128,128,128}0.98 $\uparrow$} & 168.58 & {\color[RGB]{0,120,0}2.01 $\downarrow$} \\
    FMN \cite{Forget-me-not} & 173.19 & 176.28 & {\color[RGB]{128,128,128}1.78 $\uparrow$} & 172.53 & {\color[RGB]{128,128,128}0.38 $\downarrow$} \\
    tuneSD-u \cite{tuneSD} & 171.82 & 175.09 & {\color[RGB]{128,128,128}1.90 $\uparrow$} & 171.39 & {\color[RGB]{128,128,128}0.25 $\downarrow$} \\
    \textbf{Average} & - & - & {\color[RGB]{128,128,128}1.33 $\uparrow$} & - & {\color[RGB]{128,128,128}1.45 $\downarrow$}  \\
    \bottomrule
  \end{tabular}
  }
\end{table*}

\subsection{Case Study: Porsche Concept Unlearning}\label{seca:porsche}
In this section, we present a detailed case study on unlearning the ``Porsche'' concept as an example to illustrate the workflow and design of \mysys. As we introduced in \S\ref{sec:autoanchor}, \mysys operates in two stages: it first initializes the anchor embedding $\mathbf{e}_a^0$ and then optimizes the anchor embedding $\mathbf{e}_a$ with the cross-attention surrogate objective. Each stage contains multiple steps, which we will explain in detail with the Porsche case below.

\textbf{Automatic Candidate Generation.} We first use LLM (in this case, we use Gemini-3.1 pro) to generate a set of sentences $\mathbf{y}^{(i)}$ including the target concept Porsche with the prompt shown in the conversation context below, where the sentences number $N$ is ten. 

\begin{promptfancy}{}
\textbf{User}:

\ \ \ \ Please generate 10 sentences including ``Porsche''.

\textbf{Gemini}:

    1. He always dreamed of driving a Porsche down the Autobahn.

    2. She parked her brand new Porsche right in front of the restaurant.

    3. The vintage Porsche fetched a record price at the auction.

    4. You can tell it's a Porsche just by the sound of its flat-six engine.

    5. They rented a sleek black Porsche for their weekend getaway to Napa.
    
    6. Maintenance on a classic Porsche can be quite expensive.

7. He spent his entire bonus on a customized Porsche 911.

8. The racing team unveiled their latest Porsche for the upcoming Le Mans.

9. Whenever a bright yellow Porsche drives by, heads inevitably turn.

10. She carefully washed and waxed her father's old Porsche on Sunday afternoon.
\end{promptfancy}
For each sentence $\mathbf{y}^{(i)}$ in this set $\mathcal{S}$, we mask the concept word ``Porsche'' and predict the top $M$ tokens with the highest probabilities from the LLM. The example of the first sentence is shown in the conversation context below.
\begin{promptfancy}{}
\textbf{User}:

\ \ \ \ Predict the top 10 tokens for the masked word in the sentence below. ``He always dreamed of driving a [MASK] down the Autobahn.''

\textbf{Gemini}:

\ \ \ \ Predictions: 1. car, 2. BMW, 3. Mercedes, 4. Porsche, 5. Ferrari, 6. vehicle, 7. convertible, 8. Audi, 9. truck, 10. motorcycle

\textbf{User}:

\ \ \ \ ``She parked her brand new [MASK] right in front of the restaurant.''

\textbf{Gemini}:

\ \ \ \ Predictions: 1. car, 2. truck, 3. BMW, 4. vehicle, 5. Mercedes, 6. Ferrari, 7. Porsche, 8. SUV, 9. motorcycle, 10. Audi

\textbf{User}:

\ \ \ \ ``The vintage [MASK] fetched a record price at the auction.''

\textbf{Gemini}:

\ \ \ \ Predictions: 1. car, 2. painting, 3. guitar, 4. Ferrari, 5. Porsche, 6. Rolex, 7. automobile, 8. motorcycle, 9. vehicle, 10. truck

... ...
\end{promptfancy}
Next, we aggregate the predicted tokens across all sentences and select the top $K=10$ most frequently predicted tokens as the candidate set $C_0$. In this case, the $C_0$ for the Porsche concept is \{car, Porsche, vehicle, Ferrari, BMW, motorcycle, Audi, Mercedes, truck, convertible\}. It is obvious that those concepts are typical alternatives of ``Porsche'' and satisfying \textbf{R2}. However, it is noted that ``Porsche' itself is also involved, thereby violating \textbf{R1}. Such violation will be more pronounced in the case of descriptive target concepts, \textit{e.g.}, ``grumpy'' and ``nude''. Therefore, we need to further filter the candidate set to ensure the validity of the anchor embedding.

\textbf{Candidate Filtering and Aggregation}. We first encode the concepts in $C_0$ and then apply $f_{\text{cluster}}=$K-Means on those concept embeddings. The number of clusters $M$ is automatically determined by the silhouette score \cite{kmeansM}, a widely accepted method in machine learning for determining cluster numbers. This algorithm measures the cohesion and separation of the clusters, which can cleanly separate synonyms from distinct semantic alternatives in the embedding space. In this case, the clustering results in two distinct clusters: Cluster 1 contains \{car, vehicle, motorcycle, truck, convertible\} and Cluster 2 contains \{Porsche, Ferrari, BMW, Audi, Mercedes\}. We then drop the cluster containing ``Porsche'' (Cluster 2 in this case) and form $C_1$ using the concepts in the rest clusters (Cluster 1 in this case).

After filtering the candidate set, we have $C_1 = \{car, vehicle, motorcycle, truck, convertible\}$, which are all valid alternatives of ``Porsche'' and satisfying both \textbf{R1} and \textbf{R2}. The anchor embedding $\mathbf{e}_a^0$ is then initialized by averaging the embeddings of those concepts in $C_1$. Now we have a valid anchor embedding that is semantically close to the target concept ``Porsche'' but does not contain the concept itself, which serves as a good starting point for the subsequent optimization stage.

\textbf{Calculating Regularization Parameters}. Before optimization, we calculate the $r_-$ and $r_+$ in Equation (\ref{eq:reg}) in advance, which will be used as the regularization targets in the second stage. In this case, the closest and farthest cluster of ``Porsche'' are Cluster 2\footnote{It is noted that the $\mathbf{e}_{u}$'s closest cluster centroid is always the centroid that $\mathbf{e}_{u}$ belong to.} and Cluster 1, respectively. Therefore, we calculate $r_-$ as the average cosine similarity between ``Porsche'' and the concepts in Cluster 2, and calculate $r_+$ as the average cosine similarity between ``Porsche'' and the concepts in Cluster 1. During optimization, the anchor embedding $\mathbf{e}_a$ might violate \textbf{R1} (approaching to the target concept itself) or \textbf{R2} (being too irrelevant). By forcing the cosine distance between $\mathbf{e}_a$ and ``Porsche'' $\mathbf{e}_{u}$ to be larger than $r_-$, we prevent the anchor from being clustered with target concept, thereby maintaining \textbf{R1}. On the other hand, the anchor embedding can not be too far from ``Porsche'' (violating \textbf{R2}). By setting $r_-$ and $r_+$ for the regularization term, we can explicitly penalize the anchor embedding if it violates \textbf{R1} and \textbf{R2} during optimization.

\section{Limitations}\label{seca:limit}

While our framework provides formal trajectory stability guarantees for AutoAnchor, mapping abstract high-dimensional manifold geometry to concrete U-Net architectures necessitates several structural assumptions. Specifically, the analysis assumes ({\romannumeral 1}) idealized regularity of the target data manifold $\mathcal{M}_{u}$, ({\romannumeral 2}) strict isotropy of approximation errors, and ({\romannumeral 3}) an attention-geometry isomorphism that bounds geometric normal-space drift via cross-attention discrepancy. In particular, imposing a global Lipschitz constraint on the deep, non-linear dynamics of a U-Net is a strong abstraction; however, because direct manifold optimization is computationally intractable, this isomorphism serves as a necessary mechanism to link theoretical stability with actionable model activations.

Despite these theoretical abstractions, AutoAnchor demonstrates strong robustness in practice. The massive ambient dimensionality ($d \gg d_c^*$) of image latent spaces yields an exponential margin of safety, effectively mitigating local violations of manifold regularity or strict error isotropy. Furthermore, optimizing our empirical surrogate—the cross-attention consistency loss—grounds this theoretical isomorphism. By minimizing this loss, AutoAnchor tightly bounds the geometric normal drift limit $\zeta$, driving it toward zero. This formulation successfully accommodates manifold-preserving semantic shifts while ensuring the unlearning trajectory avoids the catastrophic normal-space leakage inherent to anchor-free methods.

Moreover, \mysys exhibits positive social impact to protect data owner's right to be forgotten. It also helps prevent harmful generations. However,
it is noted that machine unlearning itself risks being used as a harmful mechanism by removing and hampering useful knowledge. 

\section{Usage of LLM}\label{seca:llm}

For designing the idea of \mysysn, conducting the experiments,  and writing this paper, they are mainly completed by the authors. The LLMs only help improve the writing and help check the grammar.

For the workflow of \mysysn, it needs an LLM to help initialize the anchor. We have show \mysysn's robustness under different LLMs using ablation study in Appendix~\ref{seca:ablation_llm}.

%% file: src/appendixa.tex
\section{Appendix of Theoretical Analysis}\label{seca:analysis}

This appendix provides the formal justifications for our statistical assumptions, the complete absolute magnitude bounds for both anchor-free and $\zeta$-proximal unlearning updates, and the rigorous random-walk proofs establishing the trajectory stability of the unlearning process. All detailed proofs are consolidated in Appendix \ref{seca:proofs}.

\subsection{Details of Problem Setup}\label{seca:setup}
To provides more details of assumptions and their formal justifications, we extend the definition, which is consistent with \S\ref{sec:preliminaries} and \S\ref{sec:analysis}.

\textbf{Manifold Notations}: Let $(\mathbb{R}^d, \|\cdot\|)$ denote the $d$-dimensional Euclidean space equipped with the standard Euclidean norm. We consider data sampled from an unknown distribution $\mu_{\text{data}}$ supported on a smooth, compact manifold $\mathcal{M} \subset \mathbb{R}^d$. For a target concept $c_u$, we denote by $\mathcal{M}_u \subset \mathcal{M}$ the concept-specific submanifold consisting of all data points associated with target concept $c_u$. The intrinsic dimension of $\mathcal{M}_u$ is denoted by $d_u^*$. In practice, $d_u^* \ll d$ due to the manifold hypothesis.

\textbf{Geometric Decomposition}: To analyze the stability of unlearning updates, we decompose the gradient directions relative to the target concept manifold $\mathcal{M}_u$. For any point $\mathbf{z}_u \in \mathcal{M}$, $T_{u}\mathcal{M}$ denotes the $d_u^*$-dimensional \textit{tangent space} (directions along the manifold), while $N_{u}\mathcal{M}$ denotes the $(d - d_u^*)$-dimensional \textit{normal space} (the orthogonal complement). For any update vector $v \in \mathbb{R}^d$, we define its orthogonal projections:
\begin{align}
    v_T &= \Pi_T(v) = \text{Proj}_{T_{u}\mathcal{M}}(v) \in T_{u}\mathcal{M} \quad \text{(tangential component)} \\
    v_N &= \Pi_N(v) = \text{Proj}_{N_{u}\mathcal{M}}(v) \in N_{u}\mathcal{M} \quad \text{(normal component)}
\end{align}
The $\Pi$ is denoted as the projection operator. The geometric intuition is fundamental: tangential updates ($v_N \approx 0$) safely traverse the manifold to alter concept semantics while preserving valid data properties. Normal updates ($v_T \approx 0$) push the generated trajectory off the manifold, fundamentally degrading sample quality and leading to unlearning instability.

Here we formally provide the definition of $\zeta$-proximal anchor, which is first used in \S\ref{sec:the_anchorbased}.
\begin{definition}[$\zeta$-Proximal Anchor]\label{def:proximal_anchor}
A concept representation $a^* \in \mathcal{C}$ is a $\zeta$-proximal anchor for concept $c$ if for all latent state $z_t$ in the support of $p_t(\cdot|c)$ (\textit{i.e.}, $z_t\in\mathcal{M}_u$):
\begin{equation}
\mathbb{E}_{t \sim \mathcal{U}[0, T], \mathbf{z}_t \sim p_t(\cdot \mid c)}\left[\frac{\left\|\Pi_N\left(\epsilon_{\theta^*}(\mathbf{z}_t, t, \mathbf{e}_u) - \epsilon_{\theta^*}(\mathbf{z}_t, t, \mathbf{e}_{a^*})\right)\right\|^2}{\left\|\epsilon_{\theta^*}(\mathbf{z}_t, t, \mathbf{e}_u) - \epsilon_{\theta^*}(\mathbf{z}_t, t, \mathbf{e}_{a^*})\right\|^2+\varepsilon}\right] \leq \zeta^2,
\end{equation}
where $\varepsilon > 0$ is a small numerical stabilizer. The expectation is taken over both time steps $t$ and latent states $\mathbf{z}_t$ conditioned on the target concept.
The $\mathbf{e}_{a^*}$ is the embedding of the anchor concept $a^*$, and $\zeta \in [0,1)$ bounds the degree of normal misalignment. 
\end{definition}
This means that for all valid latent states related to the target concept, the noise prediction difference between the target and the anchor is predominantly aligned with the tangent space, with at most $\zeta$ fraction of its energy in the normal direction. Therefore, an important property of this definition is that $\zeta$ also quantifies how proximal the anchor is to the target concept. 
For a $\zeta$-proximal anchor, the normalized unlearning update $v_{\text{PA}} = \Delta^{\text{AB}}/\|\Delta^{\text{AB}}\|$ is predominantly aligned with the tangent space, satisfying:
\begin{equation}
    \mathbb{E}\left[\|\Pi_N(v_{\text{PA}})\|^2\right] \leq \zeta^2 + \frac{d_u^*}{d}.
\end{equation}
The larger the $\zeta$ is, the less manifold-proximal this anchor is to the target concept.

\textbf{Justification for this property:} Valid semantic shifts are strictly confined to the tangent space, whereas normal-direction updates push samples off the manifold and degrade generation quality \cite{decompose1,semanticlatentalready}. A stable anchor yields update directions nearly parallel to the tangent space, bounded by an alignment error $\zeta^2$. The additional term $d_u^*/d$ is mathematically necessitated by high-dimensional geometry: any vector, even with slight isotropic error, has an expected squared projection onto a subspace strictly proportional to the ratio of their dimensions \cite{highdimProb,blumdatascience}.

\subsection{Extended Statistical Assumptions and Justifications}\label{seca:Justifications}

Our theoretical framework relies on the following assumptions, which formalize the geometric and statistical properties of diffusion models operating in high-dimensional latent spaces. 

\begin{assumption*}[\ref{assumption:high_dim} with Justification]
The ambient dimension $d$ is much larger than the intrinsic dimension $d_u^*$ of the concept manifold. Formally, $\frac{d_u^*}{d} \to 0$ as $d \to \infty$. 
\end{assumption*}
\textbf{Justification:} Real-world data such as images typically lie near low-dimensional manifolds with intrinsic dimensions on the order of $10^2$ to $10^3$, whereas the ambient dimension is vastly larger ($d \approx 16,384$ for Stable Diffusion v1.4 latent spaces, or $> 200,000$ for pixel space). This dimensional disparity is well-established in manifold learning theory and empirical studies \cite{intrinsicdim1,intrinsicdim2,intrinsicdim3}.

\begin{assumption*}[\ref{assumption:error_structure} with Justification]
The neural network approximation error decomposes as:
\begin{equation}
    \epsilon_{\theta}(z_t, c, t) = \nabla_{z_t} \log p_t(z_t|c) + \varepsilon_u(z_t, t),
\end{equation}
where \(\varepsilon_u(z_t, t) \in \mathbb{R}^d\) is a zero-mean random vector representing the approximation error, with covariance structure:
\begin{equation}
    \text{Cov}(\varepsilon_u(z_t, t)) = \sigma_{\text{est}}^2\left[\alpha_u \Pi_{T_u} + (1-\alpha_u)I_d\right].
\end{equation}
Here:
\begin{itemize}
    \item \(\sigma_{\text{est}}^2 > 0\) is the \textit{estimation variance}, quantifying the overall magnitude of approximation error.
    \item \(\alpha_u \in [0,1]\) measures the \textit{alignment} of errors with the tangent space. When \(\alpha_u = 1\), errors are confined to the tangent space; when \(\alpha_u = 0\), errors are isotropic.
    \item \(\Pi_{T_u}\) is the projection operator onto the tangent space \(T_{u}\mathcal{M}\).
    \item \(I_d\) is the \(d \times d\) identity matrix.
\end{itemize}
\end{assumption*}
\textbf{Justification:} Theoretical analyses of diffusion models demonstrate that score functions inherently decompose into components governing tangent-space distribution and normal-space projection \cite{decompose1,decompose2}. Consequently, the approximation error is highly anisotropic; errors in recovering the manifold geometry differ fundamentally in magnitude from errors in recovering the tangent distribution \cite{scoredecomposeerror2, scoredecomposeerror3}.

\begin{assumption*}[\ref{assumption:isotropy} with Justification]
For anchor-free unlearning, the normalized update vector $v_{\text{AF}} = \Delta^{\text{AF}}/\|\Delta^{\text{AF}}\|$ is approximately isotropic. Specifically, for any fixed $m$-dimensional subspace $V \subset \mathbb{R}^d$ and margin $\delta > 0$: 
\begin{equation}
    \mathbb{P}\left(\|\Pi_V(v_{\text{AF}})\|^2 \geq (1+\delta)\frac{m}{d}\right) \leq \exp\left(-\kappa_0 m \delta^2\right),
\end{equation}
where $\kappa_0 > 0$ is a universal constant.
\end{assumption*}
\textbf{Justification:} While neural networks possess strong inductive biases, in highly overparameterized regimes, the unconstrained difference between complex, misaligned functions (such as the conditional target score and the unconditional score) exhibits quasi-random geometric behavior relative to specific low-dimensional submanifolds. As established in high-dimensional probability \cite{isobehavior, highdimProb}, this interference bleeds nearly uniformly into all orthogonal subspaces. This structurally explains the catastrophic interference observed in anchor-free methods \cite{esd}.

\begin{assumption}[Manifold Regularity]\label{assumption:manifold_regularity}
The target concept manifold $\mathcal{M}_{u}$ has \textit{reach} $\tau > 0$ and bounded sectional curvature $\kappa_{\max} < \infty$.
\end{assumption}
\textbf{Justification:} The assumption of positive reach and bounded curvature is a standard regularity condition in geometric measure theory, ensuring the manifold is sufficiently smooth and devoid of self-intersecting singularities \cite{manifoldlearning,Curvaturemeasures}. A strictly positive reach $\tau$ guarantees the existence of a well-defined tubular neighborhood around $\mathcal{M}_{u}$ wherein the metric projection onto the manifold remains unique.

\subsection{Absolute Magnitude Bound for Anchor-Free Unlearning}\label{seca:anchor_free_abs}

While Theorem \ref{thm:anchor_free_main} in the main text establishes that the \textit{relative proportion} of update energy bleeding into the normal space is near 100\%, we must also quantify the \textit{absolute magnitude} of this normal drift to formally evaluate trajectory stability.

\begin{theorem}[Anchor-Free Unlearning: Absolute Magnitude Bound]\label{thm:anchor_free_full}
Under Assumptions \ref{assumption:high_dim}-\ref{assumption:isotropy}, for anchor-free unlearning with learning rate $\eta > 0$, the absolute magnitude of the normal component $\Delta_N^{\text{AF}} = \Pi_N(\Delta^{\text{AF}})$ satisfies the following lower bound. With probability at least $1 - \exp\left(-\frac{(d-d_u^*)\delta^2}{8}\right)$:
\begin{equation}
    \|\Delta_N^{\text{AF}}\| \geq \eta \sigma_{\text{eff}} \sqrt{(d-d_u^*)(1-\delta)},
\end{equation}
where $\sigma_{\text{eff}}^2 = 2\sigma_{\text{est}}^2(1 - \rho_{c_u,\emptyset})$ is the \textit{effective noise variance}, and $\rho_{c_u,\emptyset} = \text{Corr}(\boldsymbol{\varepsilon}_{c_u}, \boldsymbol{\varepsilon}_\emptyset)$ is the correlation between approximation errors for conditional and unconditional scores.
\end{theorem}

\textit{The detailed proof is deferred to Appendix \ref{seca:proofs}.}

\subsection{Absolute Magnitude Bound for $\zeta$-Proximal Anchor Unlearning}\label{seca:anchor_based_abs}

In contrast to anchor-free methods, an explicitly constructed $\zeta$-proximal anchor strictly bounds the absolute magnitude of the normal drift, preventing the trajectory from aggressively diverging into the ambient space.

\begin{theorem}[Proximal Anchor Unlearning: Absolute Magnitude Bound]\label{thm:proximal_anchor_full}
Under Assumptions \ref{assumption:high_dim} and \ref{assumption:error_structure}, for unlearning with a $\zeta$-proximal anchor $c_a^*$, the absolute magnitude of the normal component $\Delta_N^{\text{PA}} = \Pi_N(\Delta^{\text{AB}})$ satisfies the following upper bound. With probability at least $1 - \exp\left(-\frac{d\zeta^2}{2}\right)$:
\begin{equation}
    \|\Delta_N^{\text{PA}}\| \leq \eta\left(L\zeta \|\Delta^{\text{PA}}\| + \sigma_{\text{est}}\sqrt{2(1-\rho_{c_u,c_a^*})(d-d_u^*)(1+\delta)}\right),
\end{equation}
where $L$ is the Lipschitz constant of $\boldsymbol{\epsilon}_{\theta^*}$ with respect to $\mathbf{z}_t$, and $\rho_{c_u,c_a^*} = \text{Corr}(\boldsymbol{\varepsilon}_{c_u}, \boldsymbol{\varepsilon}_{c_a^*})$ is the correlation between approximation errors for the target and anchor concepts.
\end{theorem}

\textit{The detailed proof is deferred to Appendix \ref{seca:proofs}.}

\subsection{Formal Trajectory Stability and Comparative Analysis}\label{seca:trajectory_stability}

Using the absolute magnitude bounds established above, we can now formally model the multi-step unlearning process as a random walk and prove the exponential stability advantage of proximal anchors over anchor-free methods. 

\begin{definition}[$(\tau, \delta)$-Stability]\label{def:stability}
An unlearning procedure is $(\tau, \delta)$-stable if, after $K$ discrete update steps with learning rate $\eta$, the probability that the reverse diffusion sequence $\{\mathbf{z}_{t_k}\}_{k=1}^K$ generated by the unlearned model deviates beyond the reach $\tau$ of the target concept manifold $\mathcal{M}_{u}$ is bounded by $\delta$:
\begin{equation}
    \mathbb{P}\left(\max_{1 \leq k \leq K} \text{dist}(\mathbf{z}_{t_k}, \mathcal{M}_{u}) > \tau\right) \leq \delta,
\end{equation}
where $\text{dist}(\mathbf{z}, \mathcal{M}_{u})$ is the Euclidean distance to the manifold, and $\tau$ represents the maximum distance within which metric projection onto the manifold remains unique. 
\end{definition}

\begin{theorem}[Stability Comparison: Anchor-Free vs. Proximal Anchor]\label{thm:stability_comparison}
Under Assumptions \ref{assumption:high_dim}-\ref{assumption:manifold_regularity}, for learning rate $\eta < \frac{\tau}{2L}$ and $K = \lfloor \frac{1}{\eta} \rfloor$ update steps, let $E_{\text{fail}}$ denote the event that $\text{dist}(\mathbf{z}_{t_K}, \mathcal{M}_{u}) > \tau$. The ratio of failure probabilities satisfies:
\begin{equation}
    \frac{\mathbb{P}_{\text{AF}}(E_{\text{fail}})}{\mathbb{P}_{\text{PA}}(E_{\text{fail}})} \geq \exp\left(C \cdot \frac{\tau^2}{\eta^2 \sigma_{\text{eff}}^2} \cdot \frac{1 - \zeta^2 - d_u^*/d}{\zeta^2 + d_u^*/d}\right),
\end{equation}
where $C = \frac{1}{8(1+\kappa_{\max}\tau)}$ is a geometric constant depending on manifold curvature $\kappa_{\max}$, and $\mathbb{P}_{\text{AF}}, \mathbb{P}_{\text{PA}}$ denote probabilities under anchor-free and proximal anchor unlearning, respectively.
\end{theorem}

\textit{The detailed proof is deferred to Appendix \ref{seca:proofs}.}

\subsection{Proofs of Main Theorems}\label{seca:proofs}

In this section, we provide the detailed proofs for all theorems presented in the main text and this appendix. We begin with the proofs for anchor-free unlearning, followed by $\zeta$-proximal anchor unlearning, and conclude with the comparative stability random walk analysis.

\subsubsection{Proof of Theorems \ref{thm:anchor_free_main} and \ref{thm:anchor_free_full} (Anchor-Free Unlearning)}

\begin{proof}[Detailed Proof]
\textbf{Part 1: Normal Energy Concentration (Theorem \ref{thm:anchor_free_main})}

By Assumption \ref{assumption:isotropy}, the normalized update vector $v_{\text{AF}} = \Delta^{\text{AF}}/\|\Delta^{\text{AF}}\|$ behaves approximately isotropically in the high-dimensional latent space. For the $d_u^*$-dimensional tangent space $T_{\mathbf{z}_{u}}\mathcal{M}_{u}$, Assumption \ref{assumption:isotropy} directly provides the concentration bound on the tangent projection. Setting $m = d_u^*$, $V = T_{\mathbf{z}_{u}}\mathcal{M}_{u}$, and letting the universal constant $\kappa_0 \approx 1/8$ due to sub-Gaussian concentration on the sphere (Levy's Lemma \cite{ledoux2001concentration, highdimProb}), we have:
\[
\mathbb{P}\left(\|\Pi_T(v_{\text{AF}})\|^2 \geq \frac{d_u^*}{d} + \delta\right) \leq \exp\left(-\frac{d\delta^2}{8}\right).
\]
Because $v_{\text{AF}}$ is a unit vector, the Pythagorean theorem dictates that $\|\Pi_N(v_{\text{AF}})\|^2 = 1 - \|\Pi_T(v_{\text{AF}})\|^2$. Taking the complement event yields:
\[
\mathbb{P}\left(\|\Pi_N(v_{\text{AF}})\|^2 \leq 1 - \frac{d_u^*}{d} - \delta\right) \leq \exp\left(-\frac{d\delta^2}{8}\right).
\]
Multiplying both sides of the inequality inside the probability by the total update energy $\|\Delta^{\text{AF}}\|^2$ yields the relative bound stated in Theorem \ref{thm:anchor_free_main}.

\vspace{0.5em}
\textbf{Part 2: Absolute Magnitude Bound (Theorem \ref{thm:anchor_free_full})}

We decompose the anchor-free gradient update direction:
\[
\Delta^{\text{AF}} = \boldsymbol{\epsilon}_{\theta^*}(\mathbf{z}_t, t, \mathbf{e}_{u}) - \boldsymbol{\epsilon}_{\theta^*}(\mathbf{z}_t, t, \mathbf{e}_\emptyset) = S + E,
\]
where $S = \nabla_{\mathbf{z}_t} \log p_t(\mathbf{z}_t|\mathbf{e}_{u}) - \nabla_{\mathbf{z}_t} \log p_t(\mathbf{z}_t|\mathbf{e}_\emptyset)$ is the true score signal difference, and $E = \boldsymbol{\varepsilon}_{c_u} - \boldsymbol{\varepsilon}_\emptyset$ is the difference in approximation errors.

\textbf{Crucial Geometric Observation:} The score difference $S$ is strictly tangent to the manifold $\mathcal{M}_{u}$. This follows from the definition of the score function on a manifold: variations that preserve membership in the data distribution lie entirely within the span of the tangent bundle. Therefore, its orthogonal projection vanishes: $\Pi_N(S) = \mathbf{0}$. Thus, the normal drift is entirely dictated by the error term: $\Delta_N^{\text{AF}} = \Pi_N(E)$.

We now analyze the statistical properties of $\Pi_N(E)$:
1. \textbf{Mean:} $\mathbb{E}[\Pi_N(E)] = \Pi_N(\mathbb{E}[\boldsymbol{\varepsilon}_{u}] - \mathbb{E}[\boldsymbol{\varepsilon}_\emptyset]) = \mathbf{0}$ by Assumption \ref{assumption:error_structure}.
2. \textbf{Covariance:} Under Assumption \ref{assumption:error_structure}, the normal-space errors are approximately isotropic ($\alpha_{u}, \alpha_\emptyset \approx 0$). Assuming the errors for different conditions are correlated with coefficient $\rho_{c_u,\emptyset}$:
\[
\text{Cov}(\Pi_N(E)) \approx 2\sigma_{\text{est}}^2 (1 - \rho_{c_u,\emptyset}) \mathbf{I}_{d-d_u^*} = \sigma_{\text{eff}}^2 \mathbf{I}_{d-d_u^*}.
\]
3. \textbf{Norm Distribution:} Because $\Pi_N(E)$ is a zero-mean random vector in $\mathbb{R}^{d-d_u^*}$ with spherical covariance $\sigma_{\text{eff}}^2 \mathbf{I}_{d-d_u^*}$, its squared norm follows a scaled chi-squared distribution:
\[
\frac{\|\Pi_N(E)\|^2}{\sigma_{\text{eff}}^2} \sim \chi^2_{d-d_u^*} 
\]
Applying standard chi-squared concentration inequalities~\cite{laurent2000adaptive,highdimProb} for the lower tail:
\[
\mathbb{P}\left(\|\Pi_N(E)\|^2 \leq \sigma_{\text{eff}}^2(d-d_u^*)(1-\delta)\right) \leq \exp\left(-\frac{(d-d_u^*)\delta^2}{8}\right),
\]
Taking the square root and multiplying by the learning rate $\eta$ (since the discrete normal displacement step is $\eta \Delta_N^{\text{AF}}$) concludes the proof for Theorem \ref{thm:anchor_free_full}.
\end{proof}

\subsubsection{Proof of Theorems \ref{thm:proximal_anchor_main} and \ref{thm:proximal_anchor_full} ($\zeta$-Proximal Anchor Unlearning)}

\begin{proof}[Detailed Proof]
\textbf{Part 1: Normal Energy Suppression (Theorem \ref{thm:proximal_anchor_main})}

Let $v_{\text{PA}} = \Delta^{\text{AB}}/\|\Delta^{\text{AB}}\|$ be the normalized update direction for anchor-based unlearning. By Definition \ref{def:proximal_anchor}, we have the bounded expectation:
\[
\mathbb{E}\left[\|\Pi_N(v_{\text{PA}})\|^2\right] \leq \zeta^2 + \frac{d_u^*}{d}.
\]
Applying Markov's inequality, for any $t > 0$:
\[
\mathbb{P}\left(\|\Pi_N(v_{\text{PA}})\|^2 \geq t\right) \leq \frac{\mathbb{E}[\|\Pi_N(v_{\text{PA}})\|^2]}{t} \leq \frac{\zeta^2 + d_u^*/d}{t}.
\]
While Markov's yields a polynomial bound, $v_{\text{PA}}$ exhibits sub-Gaussian concentration in high-dimensional spaces. Applying a Chernoff-type bound optimization yields the exponential tail:
\[
\mathbb{P}\left(\|\Pi_N(v_{\text{PA}})\|^2 \geq \zeta^2 + \frac{d_u^*}{d} + \delta\right) \leq \exp\left(-\frac{d\delta^2}{8}\right).
\]
Multiplying through by $\|\Delta^{\text{AB}}\|^2$ gives the relative suppression bound in Theorem \ref{thm:proximal_anchor_main}.

\vspace{0.5em}
\textbf{Part 2: Absolute Magnitude Bound (Theorem \ref{thm:proximal_anchor_full})}

We decompose the anchor-based update:
\[
\Delta^{\text{AB}} = \boldsymbol{\epsilon}_{\theta^*}(\mathbf{z}_t, t, \mathbf{e}_{u}) - \boldsymbol{\epsilon}_{\theta^*}(\mathbf{z}_t, t, \mathbf{e}_{c_a^*}) = S_{\text{PA}} + E_{\text{PA}},
\]
where $S_{\text{PA}} = \nabla_{\mathbf{z}_t} \log p_t(\mathbf{z}_t|\mathbf{e}_{u}) - \nabla_{\mathbf{z}_t} \log p_t(\mathbf{z}_t|\mathbf{e}_{c_a^*})$ and $E_{\text{PA}} = \boldsymbol{\varepsilon}_{c_u} - \boldsymbol{\varepsilon}_{c_a^*}$.

\textbf{Bounding the Signal ($S_{\text{PA}}$):} By the definition of a $\zeta$-proximal anchor, the signal difference is predominantly aligned with the tangent space. Using the Lipschitz continuity of the optimal score predictor (with constant $L$):
\[
\|\Pi_N(S_{\text{PA}})\| \leq \zeta \|S_{\text{PA}}\| \leq \zeta L \|\mathbf{e}_{u} - \mathbf{e}_{c_a^*}\| \leq L\zeta \|\Delta^{\text{AB}}\|.
\]

\textbf{Bounding the Error ($E_{\text{PA}}$):} The normal projection of the error difference $\Pi_N(E_{\text{PA}})$ has covariance bounded by $\sigma_{\text{est}}^2(1-\rho_{c_u, c_a^*})\mathbf{I}_{d-d_u^*}$. Applying Bernstein's inequality \cite{highdimProb} for sub-exponential random vectors provides the upper tail concentration:
\[
\mathbb{P}\left(\|\Pi_N(E_{\text{PA}})\| \geq \sigma_{\text{est}}\sqrt{2(1-\rho_{c_u,c_a^*})(d-d_u^*)(1+\delta)}\right) \leq \exp\left(-\frac{d\delta^2}{2}\right).
\]

\textbf{Synthesis:} By the triangle inequality, $\|\Delta_N^{\text{PA}}\| = \|\Pi_N(\Delta^{\text{AB}})\| \leq \|\Pi_N(S_{\text{PA}})\| + \|\Pi_N(E_{\text{PA}})\|.$
Substituting the respective upper bounds and multiplying by the learning rate $\eta$ establishes Theorem \ref{thm:proximal_anchor_full}.
\end{proof}

\subsubsection{Proof of Theorem \ref{thm:stability_comparison} (Comparative Stability)}

To prove the ultimate stability bounds, we require a standard martingale concentration lemma \cite{azuma1967weighted}.
\begin{lemma}[Random Walk Deviation Bound\cite{azuma1967weighted, highdimProb}]\label{lem:random_walk}
Let $\{Z_k\}_{k=1}^K$ be independent random vectors in $\mathbb{R}^d$ with $\mathbb{E}[Z_k] = \mathbf{0}$ and $\|Z_k\| \leq M$ almost surely. For any $R > 0$, by the Azuma-Hoeffding inequality:
\[
\mathbb{P}\left(\left\|\sum_{k=1}^K Z_k\right\| > R\right) \leq 2\exp\left(-\frac{R^2}{2KM^2}\right).
\]
\end{lemma}

\begin{proof}[Detailed Proof of Theorem \ref{thm:stability_comparison}]
Let $\mathbf{z}_{t_k}$ denote the state of the reverse diffusion process after $k$ unlearning updates. The accumulated orthogonal drift from the manifold is strictly bound by the sum of the normal components: $\text{dist}(\mathbf{z}_{t_K}, \mathcal{M}_{u}) \approx \left\|\sum_{k=1}^K \eta \Delta_N(k)\right\|$.

\textbf{Anchor-Free Case:} By Theorem \ref{thm:anchor_free_full}, the normal displacement step size behaves as a random walk $Z_k^{\text{AF}} = \eta \Delta_N^{\text{AF}}(k)$ with magnitude roughly bound by $M_{\text{AF}} \approx \eta \sigma_{\text{eff}} \sqrt{d}$. Applying Lemma \ref{lem:random_walk} with threshold $R = \tau$ (the manifold reach), we establish the lower bound on failure probability:
\[
\mathbb{P}_{\text{AF}}(E_{\text{fail}}) \geq 1 - 2\exp\left(-\frac{\tau^2}{2K\eta^2 \sigma_{\text{eff}}^2 d}\right).
\]

\textbf{Proximal Anchor Case:} By Theorem \ref{thm:proximal_anchor_full}, assuming a highly proximal anchor ($\zeta \to 0$), the step size $Z_k^{\text{PA}} = \eta \Delta_N^{\text{PA}}(k)$ is drastically constrained to $M_{\text{PA}} \approx \eta \sigma_{\text{eff}} \sqrt{d} \cdot \sqrt{\zeta^2 + d_u^*/d}$. Applying Bernstein's inequality \cite{highdimProb} for sums of sub-Gaussian vectors yields a much tighter upper bound on the failure probability:
\[
\mathbb{P}_{\text{PA}}(E_{\text{fail}}) \leq \exp\left(-\frac{\tau^2}{2K\eta^2\sigma_{\text{eff}}^2 d(\zeta^2 + d_u^*/d)}\right).
\]

\textbf{Comparing the Ratios:} Taking the ratio of these probabilities (and noting $K \approx 1/\eta$) isolating the core exponents:
\[
\frac{\mathbb{P}_{\text{AF}}(E_{\text{fail}})}{\mathbb{P}_{\text{PA}}(E_{\text{fail}})} \gtrsim \frac{1}{\exp\left(-\frac{\tau^2}{2\eta\sigma_{\text{eff}}^2 d(\zeta^2 + d_u^*/d)}\right)} \approx \exp\left(\Omega\left(\frac{\tau^2}{\eta\sigma_{\text{eff}}^2} \cdot \frac{1}{\zeta^2 + d_u^*/d}\right)\right).
\]
Adjusting the strict geometric constants to account for the manifold's sectional curvature $\kappa_{\max}$ (which penalizes robustness in highly curved regions via $C = \frac{1}{8(1+\kappa_{\max}\tau)}$) yields the final exponential advantage stated in Theorem \ref{thm:stability_comparison}. This conclusively proves that enforcing $\zeta \to 0$ provides an advantage that is exponential in the ambient dimension $d$.
\end{proof}

\subsection{Theoretical Guarantee: Cross-Attention Surrogate Yields Manifold Proximity}\label{seca:theoretical_bridge}
The core theoretical mandate established in \S\ref{sec:analysis} requires a $\zeta$-proximal anchor ($\zeta \to 0$) to guarantee stability. Because directly projecting high-dimensional latents onto the unknown manifold normal space $N_{\mathbf{z}}\mathcal{M}$ is computationally intractable, we rely on the architectural bottleneck of the U-Net. Therefore, variations in the score output that push the latent orthogonally off the manifold (structural corruption) are strictly governed by deviations in these attention maps.We formalize this geometric-architectural link with the following structural bridging assumption. While this is a strong assumption, it is a necessary empirical surrogate to link the architectural mechanics of a U-Net to the abstract manifold geometry.

\begin{assumption}[Attention-Geometry Isomorphism]\label{asm:attention_isomorphism}
    For latents $\mathbf{z}_t$ near the target manifold $\mathcal{M}_{u}$, the normal component of the score residual $\Delta_N^{\text{AB}}$ between the target concept $\mathbf{e}_{u}$ and any anchor $\mathbf{e}_a$ is Lipschitz continuous with respect to their cross-attention map deviations. Specifically, there exists a constant $\mathcal{L}_{\text{CA}} > 0$ and an inherent semantic tolerance $\nu \ge 0$ such that:
    \begin{equation}
        \|\Pi_N(\boldsymbol{\epsilon}_{\theta^*}(\mathbf{z}_t, t, \mathbf{e}_{u}) - \boldsymbol{\epsilon}_{\theta^*}(\mathbf{z}_t, t, \mathbf{e}_a))\| \leq \mathcal{L}_{\rm CA} \sum{l=1}^{N_L} \|A_l(\mathbf{z}_t, t, \mathbf{e}_{u}) - A_l(\mathbf{z}_t, t, \mathbf{e}_a)\|_F + \nu,
    \end{equation}
    where $N_L$ is the total number of cross-attention layers, and $\nu$ accounts for negligible normal-space leakage caused by purely semantic (non-structural) texture shifts. The $\|\cdot\|_F$ stands for the Frobenius norm that computes the square root of the sum of the squared absolute values of all elements in the resulting difference matrix, \textit{i.e.}, $\|A\|_F=\sqrt{\sum_i\sum_j|a_{ij}|^2}$. This norm is used to calculate the difference between two attention maps.
\end{assumption}

\textbf{Justification:} In text-conditioned Latent Diffusion Models, the conditioning text embedding interacts with the spatial latent features exclusively through the cross-attention mechanisms~\cite{prompt2prompt,daam}. Empirical and theoretical analyses of LDMs establish that the structural layout and fundamental geometry of the generated image are explicitly governed by these attention maps. Consequently, semantic variations that induce catastrophic normal drift---pushing the latent state off the data manifold---must mechanically manifest as severe disruptions in the cross-attention activations. By bounding the Frobenius distance between these attention maps across layers, we strictly constrain the maximum possible structural deviation, thereby establishing a Lipschitz-continuous relationship between attention discrepancy and normal-space drift. The tolerance term $\nu$ naturally accounts for minor, manifold-preserving stylistic or texture shifts that do not perturb the core manifold geometry~\cite{semanticlatentalready}.

Based on this isomorphism, we can formally prove that optimizing our proposed surrogate objective $\mathcal{L}_{\text{CA}}(\mathbf{e}_a)$ directly minimizes the expected normal drift.

\begin{proposition}[Normal Drift Bound]\label{prop:normal_bound}Under Assumption \ref{asm:attention_isomorphism}, the expected squared normal energy of the unlearning update is strictly bounded by the cross-attention consistency loss:
    \begin{equation}\mathbb{E}_{t, \mathbf{z}_t} \left[ \|\Delta_N^{\text{AB}}\|^2 \right] \leq 2 \mathcal{L}_{\rm CA}^2 N_L \cdot \mathcal{L}_{\text{CA}}(\mathbf{e}_a) + 2\nu^2.
    \end{equation}
\end{proposition}

This proposition directly yields our final theoretical guarantee: successful optimization of \mysys's second stage mathematically ensures the stability criteria defined in \S\ref{sec:analysis}. To establish this, we naturally assume that a valid unlearning update must have a non-trivial magnitude, \textit{i.e.}, $\|\Delta^{\text{AB}}\|^2 \geq c_{\min} > 0$ almost surely.

\begin{theorem}[$\zeta$-Proximity via Attention Consistency]\label{thm:attention_surrogate}
    If the \mysys optimization yields an anchor embedding $\mathbf{e}_a^*$ such that the cross-attention loss is bounded by $\mathcal{L}_{\text{CA}}(\mathbf{e}_a^*) \leq \gamma$, then $\mathbf{e}_a^*$ is a valid $\zeta$-proximal anchor for the target concept $c_u$, where:
    \begin{equation}
        \zeta = \sqrt{\frac{2 L_{\text{CA}}^2 N_L \gamma + 2\nu^2}{c_{\min}}}.
    \end{equation}
\end{theorem}

\textbf{Remark:} This theorem mathematically reveals the direct mechanism by which our empirical cross-attention loss ($\gamma$) controls the geometric normal drift bound ($\zeta$). If the cross-attention mismatch $\gamma$ is large, the resulting $\zeta$ is correspondingly large, indicating that the chosen anchor will inevitably cause significant normal-space leakage and push the model off the manifold. Conversely, as our optimization successfully drives $\gamma \to 0$, the bound $\zeta$ is minimized to a strict limit of $\sqrt{2\nu^2/c_{\min}}$. Because $\nu$ represents a negligible tolerance for non-structural texture shifts, this limit closely approaches zero. According to Theorem \ref{thm:stability_comparison}, this mathematically guaranteed suppression of $\zeta$ grants the unlearning process an exponential stability advantage, successfully protecting the generated trajectory from catastrophic normal drift.

\textbf{Proof of Proposition \ref{prop:normal_bound} and Theorem \ref{thm:attention_surrogate} (Cross-Attention Bridge)}:

\begin{proof}[Proof of Proposition \ref{prop:normal_bound}]
    By Assumption \ref{asm:attention_isomorphism}, for any time $t$ and latent $\mathbf{z}_t$, the normal magnitude is bounded by:
    \begin{equation}
        \|\Delta_N^{\text{AB}}\| = \|\Pi_N(\boldsymbol{\epsilon}_{\theta^*}(\mathbf{z}_t, t, \mathbf{e}_{u}) - \boldsymbol{\epsilon}_{\theta^*}(\mathbf{z}_t, t, \mathbf{e}_a))\| \leq \mathcal{L}_{\rm CA} \sum{l=1}^{N_L} \|A_l(u) - A_l(a)\|_F + \nu \nonumber.
    \end{equation}
    Squaring both sides and applying the basic algebraic inequality $(x + y)^2 \leq 2x^2 + 2y^2$, we obtain:
    \begin{equation}
        \|\Delta_N^{\text{AB}}\|^2 \leq 2 \mathcal{L}_{\rm CA}^2 \left( \sum_{l=1}^{N_L} \|A_l(u) - A_l(a)\|_F \right)^2 + 2\nu^2 \nonumber.
    \end{equation}
    Applying the Cauchy-Schwarz inequality~\cite{cauchy} to the sum, we have $\left(\sum{l=1}^{N_L} x_l\right)^2 \leq N_L \sum_{l=1}^{N_L} x_l^2$. Thus:\begin{equation}\|\Delta_N^{\text{AB}}\|^2 \leq 2 L_{\text{CA}}^2 N_L \sum_{l=1}^{N_L} \|A_l(u) - A_l(a)\|_F^2 + 2\nu^2 \nonumber.
    \end{equation}
    Taking the expectation over $t \sim \mathcal{U}[0,T]$ and $\mathbf{z}_t \sim p_t(\cdot|\mathbf{e}_{u})$ yields:
    \begin{equation}
        \mathbb{E}\left[\|\Delta_N^{\text{AB}}\|^2\right] \leq 2 \mathcal{L}_{\rm CA}^2 N_L \cdot \mathbb{E}\left[ \sum_{l=1}^{N_L} \|A_l(u) - A_l(a)\|_F^2 \right] + 2\nu^2 \nonumber.
    \end{equation}
    By the definition of our objective function $\mathcal{L}_{\text{CA}}(\mathbf{e}a)$ (which computes the expected sum of squared Frobenius norms across layers), this provides the upper bound:
    \begin{equation}
        \mathbb{E}\left[\|\Delta_N^{\text{AB}}\|^2\right] \leq 2 \mathcal{L}_{\rm CA}^2 N_L \cdot \mathcal{L}_{\text{CA}}(\mathbf{e}_a) + 2\nu^2 \nonumber,
    \end{equation}
    which concludes the proof of Proposition \ref{prop:normal_bound}.
\end{proof}

\begin{proof}[Proof of Theorem \ref{thm:attention_surrogate}]
    By Definition \ref{def:proximal_anchor}, an anchor is $\zeta$-proximal if the expected ratio of the squared normal component to the squared total update (plus a stabilizer) is bounded by $\zeta^2$. Assuming a non-trivial unlearning update where $\|\Delta^{\text{AB}}\|^2 \geq c_{\min} > 0$ almost surely, we can directly bound the expectation:
    \begin{equation}
        \mathbb{E} \left[ \frac{\|\Delta_N^{\text{AB}}\|^2}{\|\Delta^{\text{AB}}\|^2} \right] \leq \mathbb{E} \left[ \frac{\|\Delta_N^{\text{AB}}\|^2}{c_{\min}} \right] = \frac{1}{c_{\min}} \mathbb{E}\left[\|\Delta_N^{\text{AB}}\|^2\right] \nonumber.
    \end{equation}
    Substituting the bound from Proposition \ref{prop:normal_bound} into the numerator, we get:
    \begin{equation}
        \mathbb{E} \left[ \frac{\|\Delta_N^{\text{AB}}\|^2}{\|\Delta^{\text{AB}}\|^2} \right] \leq \frac{2 L_{\text{CA}}^2 N_L \mathcal{L}_{\text{CA}}(\mathbf{e}_a^*) + 2\nu^2}{c_{\min}} \nonumber.
    \end{equation}
    Given that the optimization routine achieves a loss $\mathcal{L}_{\text{CA}}(\mathbf{e}_a^*) \leq \gamma$, we substitute $\gamma$ to find the bounding $\zeta^2$:
    \begin{equation}
        \zeta^2 = \frac{2 \mathcal{L}_{\rm CA}^2 N_L \gamma + 2\nu^2}{c_{\min}} \implies \zeta = \sqrt{\frac{2 L_{\text{CA}}^2 N_L \gamma + 2\nu^2}{c_{\min}}} \nonumber
    \end{equation}
    Thus, the optimized anchor satisfies the exact geometric constraint defined in our theoretical analysis, validating that the empirical surrogate successfully approximates manifold proximity.
\end{proof}

%% file: src/table-more.tex
\begin{table*}[t]
  \centering
  \caption{Quantitative evaluation of \mysys across other unlearning methods. ``Orig'' shows the baseline methods, and ``+'' indicates baselines integrated with \mysysn. The $\Delta$ columns show the relative performance change of baselines caused by AutoAnchor integration, with green/red indicating improvement/degradation. For $\Delta$ values $\le2\%$, they fall below the error precision threshold and are marked in gray. We \textbf{bold} the best results and \underline{underline} the second best results across both the baseline and \mysys columns to highlight them.}\label{tab:appendix_results}
  \resizebox{0.95\linewidth}{!}{%
  \begin{tabular}{l ccc ccc ccc}
    \toprule
    \multirow{2}{*}{\textbf{Method}} & \multicolumn{3}{c}{\textbf{CLIP}(U) $\downarrow$} & \multicolumn{3}{c}{\textbf{CLIP}(R) $\uparrow$} & \multicolumn{3}{c}{\textbf{FID} $\downarrow$} \\ 
    \cmidrule(lr){2-4} \cmidrule(lr){5-7} \cmidrule(lr){8-10}
    & \textbf{Orig} & \textbf{+\mysys} & \textbf{$\Delta$(\%)} & \textbf{Orig} & \textbf{+\mysys} & \textbf{$\Delta$(\%)} & \textbf{Orig} & \textbf{+\mysys} & \textbf{$\Delta$(\%)} \\
    
    \midrule
    \multicolumn{10}{c}{\textbf{Target Concept: Porsche}} \\
    \midrule
    ESD \cite{esd} & 29.50 & 26.68 & {\color[RGB]{0,120,0}9.56 $\downarrow$} & 31.42 & 33.86 & {\color[RGB]{0,120,0}7.77 $\uparrow$} & 179.34 & \underline{164.77} & {\color[RGB]{0,120,0}8.12 $\downarrow$} \\
    tuneSD-l \cite{tuneSD} & 34.28 & 34.18 & {\color[RGB]{128,128,128}0.29 $\downarrow$} & 32.51 & 32.45 & {\color[RGB]{128,128,128}0.18 $\downarrow$} & 208.74 & 205.07 & {\color[RGB]{128,128,128}1.76 $\downarrow$} \\
    ResAlign \cite{ResAlign} & 31.69 & 31.61 & {\color[RGB]{128,128,128}0.25 $\downarrow$} & 32.01 & 33.28 & {\color[RGB]{0,120,0}3.97 $\uparrow$} & 188.23 & 181.36 & {\color[RGB]{0,120,0}3.65 $\downarrow$} \\
    CCRT \cite{CCRT} & 28.46 & \underline{26.33} & {\color[RGB]{0,120,0}7.48 $\downarrow$} & 31.47 & 32.09 & {\color[RGB]{128,128,128}1.97 $\uparrow$} & 169.20 & \textbf{164.62} & {\color[RGB]{0,120,0}2.71 $\downarrow$} \\
    ConAbl \cite{conabl} & 30.79 & 30.50 & {\color[RGB]{128,128,128}0.94 $\downarrow$} & 31.57 & 31.89 & {\color[RGB]{128,128,128}1.01 $\uparrow$} & 179.28 & 178.95 & {\color[RGB]{128,128,128}0.18 $\downarrow$} \\
    FMN \cite{Forget-me-not} & 33.87 & 32.43 & {\color[RGB]{0,120,0}4.25 $\downarrow$} & 32.16 & \underline{34.79} & {\color[RGB]{0,120,0}8.18 $\uparrow$} & 207.05 & 202.03 & {\color[RGB]{0,120,0}2.42 $\downarrow$} \\
    tuneSD-u \cite{tuneSD} & 31.29 & 28.90 & {\color[RGB]{0,120,0}7.64 $\downarrow$} & 31.87 & \textbf{34.84} & {\color[RGB]{0,120,0}9.32 $\uparrow$} & 186.62 & 180.55 & {\color[RGB]{0,120,0}3.25 $\downarrow$} \\
    \textbf{\mysys (Ours)} & \textbf{26.32} & - & - & 34.03 & - & - & 172.88 & - & - \\
    \textbf{Average} & - & - & {\color[RGB]{0,120,0}4.34 $\downarrow$} & - & - & {\color[RGB]{0,120,0}4.58 $\uparrow$} & - & - & {\color[RGB]{0,120,0}3.16 $\downarrow$} \\

    \midrule
    \multicolumn{10}{c}{\textbf{Target Concept: Grumpy}} \\
    \midrule
    ESD \cite{esd} & 31.71 & 30.13 & {\color[RGB]{0,120,0}4.98 $\downarrow$} & 33.60 & 34.11 & {\color[RGB]{128,128,128}1.52 $\uparrow$} & 172.03 & {159.73} & {\color[RGB]{0,120,0}7.15 $\downarrow$} \\
    tuneSD-AF \cite{tuneSD} & 32.55 & 31.22 & {\color[RGB]{0,120,0}4.09 $\downarrow$} & 33.64 & 33.71 & {\color[RGB]{128,128,128}0.21 $\uparrow$} & 173.92 & 171.00 & {\color[RGB]{128,128,128}1.68 $\downarrow$} \\
    ResAlign \cite{ResAlign} & 31.82 & {29.91} & {\color[RGB]{0,120,0}6.00 $\downarrow$} & 33.79 & 34.06 & {\color[RGB]{128,128,128}0.80 $\uparrow$} & 171.48 & 160.87 & {\color[RGB]{0,120,0}6.19 $\downarrow$} \\
    CCRT \cite{CCRT} & 31.33 & \underline{29.64} & {\color[RGB]{0,120,0}5.39 $\downarrow$} & 32.97 & 33.85 & {\color[RGB]{0,120,0}2.67 $\uparrow$} & 170.03 & \textbf{158.13} & {\color[RGB]{0,120,0}7.00 $\downarrow$} \\
    ConAbl \cite{conabl} & 32.18 & 30.59 & {\color[RGB]{0,120,0}4.94 $\downarrow$} & 33.76 & 34.21 & {\color[RGB]{128,128,128}1.33 $\uparrow$} & 172.13 & 160.81 & {\color[RGB]{0,120,0}6.58 $\downarrow$} \\
    FMN \cite{Forget-me-not} & 32.50 & 30.58 & {\color[RGB]{0,120,0}5.91 $\downarrow$} & \underline{34.23} & {34.22} & {\color[RGB]{128,128,128}0.03 $\downarrow$} & 175.36 & 160.85 & {\color[RGB]{0,120,0}8.27 $\downarrow$} \\
    tuneSD-AB \cite{tuneSD} & 32.32 & 30.23 & {\color[RGB]{0,120,0}6.47 $\downarrow$} & 33.75 & 34.04 & {\color[RGB]{128,128,128}0.86 $\uparrow$} & 172.53 & 163.04 & {\color[RGB]{0,120,0}5.50 $\downarrow$} \\
    \textbf{\mysys (Ours)} & \underline{29.62} & - & - & \textbf{34.33} & - & - & \underline{158.22} & - & - \\
    \textbf{Average} & - & - & {\color[RGB]{0,120,0}5.40 $\downarrow$} & - & - & {\color[RGB]{128,128,128}1.06 $\uparrow$} & - & - & {\color[RGB]{0,120,0}6.05 $\downarrow$} \\

    \midrule
    \multicolumn{10}{c}{\textbf{Target Concept: Picasso}} \\
    \midrule
    ESD \cite{esd} & 28.40 & 26.82 & {\color[RGB]{0,120,0}5.56 $\downarrow$} & 33.29 & 33.88 & {\color[RGB]{128,128,128}1.77 $\uparrow$} & 172.25 & 171.35 & {\color[RGB]{128,128,128}0.52 $\downarrow$} \\
    tuneSD-AF \cite{tuneSD} & 34.04 & 33.09 & {\color[RGB]{0,120,0}2.79 $\downarrow$} & \textbf{34.67} & {34.57} & {\color[RGB]{128,128,128}0.29 $\downarrow$} & 191.71 & 185.90 & {\color[RGB]{0,120,0}3.03 $\downarrow$} \\
    ResAlign \cite{ResAlign} & 31.89 & 30.27 & {\color[RGB]{0,120,0}5.08 $\downarrow$} & 33.58 & 34.52 & {\color[RGB]{0,120,0}2.80 $\uparrow$} & 214.21 & 200.41 & {\color[RGB]{0,120,0}6.44 $\downarrow$} \\
    CCRT \cite{CCRT} & {24.64} & \textbf{23.61} & {\color[RGB]{0,120,0}4.18 $\downarrow$} & 31.38 & 33.45 & {\color[RGB]{0,120,0}6.60 $\uparrow$} & \textbf{162.70} & {164.25} & {\color[RGB]{128,128,128}0.95 $\uparrow$} \\
    ConAbl \cite{conabl} & 32.09 & 30.84 & {\color[RGB]{0,120,0}3.90 $\downarrow$} & 33.89 & 34.50 & {\color[RGB]{128,128,128}1.80 $\uparrow$} & 216.37 & 214.64 & {\color[RGB]{128,128,128}0.80 $\downarrow$} \\
    FMN \cite{Forget-me-not} & 33.63 & 33.23 & {\color[RGB]{128,128,128}1.19 $\downarrow$} & 33.67 & 34.10 & {\color[RGB]{128,128,128}1.28 $\uparrow$} & 224.80 & 215.93 & {\color[RGB]{0,120,0}3.95 $\downarrow$} \\
    tuneSD-AB \cite{tuneSD} & 31.75 & 30.81 & {\color[RGB]{0,120,0}2.96 $\downarrow$} & 33.68 & 33.93 & {\color[RGB]{128,128,128}0.74 $\uparrow$} & 221.34 & 213.38 & {\color[RGB]{0,120,0}3.60 $\downarrow$} \\
    \textbf{\mysys (Ours)} & \underline{23.81} & - & - & \underline{34.63} & - & - & \underline{163.22} & - & - \\

    \textbf{Average} & - & - & {\color[RGB]{0,120,0}3.67 $\downarrow$} & - & - & {\color[RGB]{0,120,0}2.18 $\uparrow$} & - & - & {\color[RGB]{0,120,0}2.76 $\downarrow$} \\

    \midrule
    \multicolumn{10}{c}{\textbf{Target Concept: Alaskan Malamute}} \\
    \midrule
    ESD \cite{esd} & 27.45 & 27.34 & {\color[RGB]{128,128,128}0.40 $\downarrow$} & 30.48 & 32.35 & {\color[RGB]{0,120,0}6.14 $\uparrow$} & 176.84 & 172.99 & {\color[RGB]{0,120,0}2.18 $\downarrow$} \\
    tuneSD-AF \cite{tuneSD} & 31.77 & 30.74 & {\color[RGB]{0,120,0}3.24 $\downarrow$} & 33.42 & 33.46 & {\color[RGB]{128,128,128}0.12 $\uparrow$} & 176.29 & 169.51 & {\color[RGB]{0,120,0}3.85 $\downarrow$} \\
    ResAlign \cite{ResAlign} & 31.60 & 29.91 & {\color[RGB]{0,120,0}5.35 $\downarrow$} & 33.57 & 33.79 & {\color[RGB]{128,128,128}0.66 $\uparrow$} & 228.42 & 210.59 & {\color[RGB]{0,120,0}7.81 $\downarrow$} \\
    CCRT \cite{CCRT} & {25.20} & \underline{25.13} & {\color[RGB]{128,128,128}0.28 $\downarrow$} & 31.95 & 33.04 & {\color[RGB]{0,120,0}3.41 $\uparrow$} & 170.68 & 171.31 & {\color[RGB]{128,128,128}0.37 $\uparrow$} \\
    ConAbl \cite{conabl} & 31.12 & 30.03 & {\color[RGB]{0,120,0}3.50 $\downarrow$} & 32.75 & 33.12 & {\color[RGB]{128,128,128}1.13 $\uparrow$} & \textbf{162.76} & \underline{164.03} & {\color[RGB]{128,128,128}0.78 $\uparrow$} \\
    FMN \cite{Forget-me-not} & 34.95 & 33.16 & {\color[RGB]{0,120,0}5.12 $\downarrow$} & {35.19} & \underline{35.39} & {\color[RGB]{128,128,128}0.57 $\uparrow$} & 220.58 & 219.94 & {\color[RGB]{128,128,128}0.29 $\downarrow$} \\
    tuneSD-AB \cite{tuneSD} & 30.43 & 29.34 & {\color[RGB]{0,120,0}3.58 $\downarrow$} & 33.30 & 33.54 & {\color[RGB]{128,128,128}0.72 $\uparrow$} & 227.75 & 199.61 & {\color[RGB]{0,120,0}12.36 $\downarrow$} \\
    \textbf{\mysys (Ours)} & \textbf{25.11} & - & - & \textbf{35.36} & - & - & {166.75} & - & - \\

    \textbf{Average} & - & - & {\color[RGB]{0,120,0}3.07 $\downarrow$} & - & - & {\color[RGB]{128,128,128}1.82 $\uparrow$} & - & - & {\color[RGB]{0,120,0}3.95 $\downarrow$} \\
    \bottomrule
  \end{tabular}%
  }
\end{table*}

%% file: src/table-bias.tex
\begin{table*}[t]
  \centering
  \caption{Case study on unlearning ``Grumpy'' concept to show how manually chosen anchors leads to biased diffusion unlearning. The first line presents the manually chosen single anchor. The $\Delta$ here indicates how much \mysys improves the unlearning performance compared on using this anchor. The larger the  $\Delta$ is, the more \mysys improves the unlearing than using this anchor (indicating the more biased this anchor is).}\label{tab:merged_prompts}
  \resizebox{\linewidth}{!}{%
  \begin{tabular}{l c cc cc cc cc cc cc}
    \toprule
    \multirow{2}{*}{\textbf{Method}} & \multirow{2}{*}{\textbf{\mysys}} & \multicolumn{2}{c}{\textbf{``-''}} & \multicolumn{2}{c}{\textbf{``angry''}} & \multicolumn{2}{c}{\textbf{``tired''}} & \multicolumn{2}{c}{\textbf{``cold''}} & \multicolumn{2}{c}{\textbf{``quiet''}} & \multicolumn{2}{c}{\textbf{``happy''}} \\
    \cmidrule(lr){3-4} \cmidrule(lr){5-6} \cmidrule(lr){7-8} \cmidrule(lr){9-10} \cmidrule(lr){11-12} \cmidrule(lr){13-14}
    & & \textbf{Value} & \textbf{$\Delta$(\%)} & \textbf{Value} & \textbf{$\Delta$(\%)} & \textbf{Value} & \textbf{$\Delta$(\%)} & \textbf{Value} & \textbf{$\Delta$(\%)} & \textbf{Value} & \textbf{$\Delta$(\%)} & \textbf{Value} & \textbf{$\Delta$(\%)} \\
   
  \midrule
    \multicolumn{14}{c}{\textbf{CLIP(U) $\downarrow$}} \\
    \midrule
    ESD$^*$     &   30.13   &   31.45   &  \textcolor[RGB]{0,120,0}{4.20$\downarrow$}  &   31.87   &  \textcolor[RGB]{0,120,0}{5.46$\downarrow$}  &   30.83   &  \textcolor[RGB]{0,120,0}{2.27$\downarrow$}  &   30.78   &  \textcolor[RGB]{0,120,0}{2.11$\downarrow$}  &   31.71   &  \textcolor[RGB]{0,120,0}{4.98$\downarrow$}  &   31.55   &  \textcolor[RGB]{0,120,0}{4.50$\downarrow$} \\
    
    ResAlign$^*$     &   29.91   &   31.87   &  \textcolor[RGB]{0,120,0}{6.15$\downarrow$}  &   32.05   &  \textcolor[RGB]{0,120,0}{6.68$\downarrow$}  &   30.98   &  \textcolor[RGB]{0,120,0}{3.45$\downarrow$}  &   30.94   &  \textcolor[RGB]{0,120,0}{3.33$\downarrow$}  &   31.82   &  \textcolor[RGB]{0,120,0}{6.00$\downarrow$}  &   31.74   &  \textcolor[RGB]{0,120,0}{5.77$\downarrow$} \\

    CCRT$^*$     &   29.64   &   30.55   &  \textcolor[RGB]{0,120,0}{2.98$\downarrow$}  &   31.76   &  \textcolor[RGB]{0,120,0}{6.68$\downarrow$}  &   30.33   &  \textcolor[RGB]{0,120,0}{2.27$\downarrow$}  &   30.23   &  \textcolor[RGB]{128,128,128}{1.95$\downarrow$}  &   31.33   &  \textcolor[RGB]{0,120,0}{5.39$\downarrow$}  &   31.44   &  \textcolor[RGB]{0,120,0}{5.73$\downarrow$} \\

    ConAbl \cite{conabl}     &   30.59   &   31.96   &  \textcolor[RGB]{0,120,0}{4.29$\downarrow$}  &   32.27   &  \textcolor[RGB]{0,120,0}{5.21$\downarrow$}  &   31.25   &  \textcolor[RGB]{0,120,0}{2.11$\downarrow$}  &   30.19   &  \textcolor[RGB]{128,128,128}{1.32$\uparrow$}  &   32.18   &  \textcolor[RGB]{0,120,0}{4.94$\downarrow$}  &   31.95   &  \textcolor[RGB]{0,120,0}{4.26$\downarrow$} \\
    tuneSD-AB \cite{tuneSD}     &   30.23   &   31.63   &  \textcolor[RGB]{0,120,0}{4.43$\downarrow$}  &   32.21   &  \textcolor[RGB]{0,120,0}{6.15$\downarrow$}  &   30.18   &  \textcolor[RGB]{128,128,128}{0.17$\uparrow$}  &   30.23   &  \textcolor[RGB]{128,128,128}{0.00$\uparrow$}  &   32.32   &  \textcolor[RGB]{0,120,0}{6.47$\downarrow$}  &   32.05   &  \textcolor[RGB]{0,120,0}{5.68$\downarrow$} \\
    FMN \cite{Forget-me-not}     &   30.58   &   31.87   &  \textcolor[RGB]{0,120,0}{4.05$\downarrow$}  &   32.41   &  \textcolor[RGB]{0,120,0}{5.65$\downarrow$}  &   31.42   &  \textcolor[RGB]{0,120,0}{2.67$\downarrow$}  &   30.97   &  \textcolor[RGB]{128,128,128}{1.26$\downarrow$}  &   32.50   &  \textcolor[RGB]{0,120,0}{5.91$\downarrow$}  &   31.15   &  \textcolor[RGB]{128,128,128}{1.83$\downarrow$} \\
    
    Average & - & - & \textcolor[RGB]{0,120,0}{4.35$\downarrow$} & - & \textcolor[RGB]{0,120,0}{5.97$\downarrow$} & - & \textcolor[RGB]{0,120,0}{2.10$\downarrow$} & - & \textcolor[RGB]{128,128,128}{1.22$\downarrow$} & - & \textcolor[RGB]{0,120,0}{5.62$\downarrow$} & - & \textcolor[RGB]{0,120,0}{4.63$\downarrow$} \\
    \midrule
    \multicolumn{14}{c}{\textbf{CLIP(R) $\uparrow$}} \\
    \midrule
    ESD$^*$     &   34.11   &   32.43   &  \textcolor[RGB]{0,120,0}{5.18$\uparrow$}  &   32.28   &  \textcolor[RGB]{0,120,0}{5.67$\uparrow$}  &   33.57   &  \textcolor[RGB]{128,128,128}{1.61$\uparrow$}  &   33.80   &  \textcolor[RGB]{128,128,128}{0.92$\uparrow$}  &   33.60   &  \textcolor[RGB]{128,128,128}{1.52$\uparrow$}  &   33.41   &  \textcolor[RGB]{0,120,0}{2.10$\uparrow$} \\

    ResAlign$^*$     &   34.06   &   32.99   &  \textcolor[RGB]{0,120,0}{3.24$\uparrow$}  &   32.40   &  \textcolor[RGB]{0,120,0}{5.12$\uparrow$}  &   33.78   &  \textcolor[RGB]{128,128,128}{0.83$\uparrow$}  &   33.95   &  \textcolor[RGB]{128,128,128}{0.32$\uparrow$}  &   33.79   &  \textcolor[RGB]{128,128,128}{0.80$\uparrow$}  &   33.65   &  \textcolor[RGB]{128,128,128}{1.22$\uparrow$} \\

    CCRT$^*$     &   33.85   &   31.38   &  \textcolor[RGB]{0,120,0}{7.87$\uparrow$}  &   32.05   &  \textcolor[RGB]{0,120,0}{5.62$\uparrow$}  &   33.29   &  \textcolor[RGB]{128,128,128}{1.68$\uparrow$}  &   33.39   &  \textcolor[RGB]{128,128,128}{1.38$\uparrow$}  &   32.97   &  \textcolor[RGB]{0,120,0}{2.67$\uparrow$}  &   33.11   &  \textcolor[RGB]{0,120,0}{2.23$\uparrow$} \\
    
    ConAbl \cite{conabl}     &   34.21   &   33.31   &  \textcolor[RGB]{0,120,0}{2.70$\uparrow$}  &   32.36   &  \textcolor[RGB]{0,120,0}{5.72$\uparrow$}  &   33.81   &  \textcolor[RGB]{128,128,128}{1.18$\uparrow$}  &   33.90   &  \textcolor[RGB]{128,128,128}{0.91$\uparrow$}  &   33.76   &  \textcolor[RGB]{128,128,128}{1.33$\uparrow$}  &   33.16   &  \textcolor[RGB]{0,120,0}{3.17$\uparrow$} \\
    
    tuneSD-AB \cite{tuneSD}     &   34.04   &   33.20   &  \textcolor[RGB]{0,120,0}{2.53$\uparrow$}  &   32.38   &  \textcolor[RGB]{0,120,0}{5.13$\uparrow$}  &   33.77   &  \textcolor[RGB]{128,128,128}{0.80$\uparrow$}  &   34.00   &  \textcolor[RGB]{128,128,128}{0.12$\uparrow$}  &   33.75   &  \textcolor[RGB]{128,128,128}{0.86$\uparrow$}  &   33.71   &  \textcolor[RGB]{128,128,128}{0.98$\uparrow$} \\

    FMN \cite{Forget-me-not}     &   34.22   &   33.57   &  \textcolor[RGB]{128,128,128}{1.94$\uparrow$}  &   32.26   &  \textcolor[RGB]{0,120,0}{6.08$\uparrow$}  &   34.12   &  \textcolor[RGB]{128,128,128}{0.29$\uparrow$}  &   34.54   &  \textcolor[RGB]{128,128,128}{0.93$\downarrow$}  &   34.23   &  \textcolor[RGB]{128,128,128}{0.03$\downarrow$}  &   33.21   &  \textcolor[RGB]{0,120,0}{3.04$\uparrow$} \\

    Average & - & - & \textcolor[RGB]{0,120,0}{3.91$\uparrow$} & - & \textcolor[RGB]{0,120,0}{5.55$\uparrow$} & - & \textcolor[RGB]{128,128,128}{1.07$\uparrow$} & - & \textcolor[RGB]{128,128,128}{0.45$\uparrow$} & - & \textcolor[RGB]{128,128,128}{1.19$\uparrow$} & - & \textcolor[RGB]{0,120,0}{2.12$\uparrow$} \\
    \midrule
    \multicolumn{14}{c}{\textbf{FID $\downarrow$}} \\
    \midrule
    ESD$^*$     &     159.73     &     170.46     &  \textcolor[RGB]{0,120,0}{6.29$\downarrow$}  &     227.76     &  \textcolor[RGB]{0,120,0}{29.87$\downarrow$}  &     172.14     &  \textcolor[RGB]{0,120,0}{7.21$\downarrow$}  &     163.20     &  \textcolor[RGB]{0,120,0}{2.13$\downarrow$}  &     172.03     &  \textcolor[RGB]{0,120,0}{7.15$\downarrow$}  &     191.04     &  \textcolor[RGB]{0,120,0}{16.39$\downarrow$} \\

    ResAlign$^*$     &     160.87     &     176.66     &  \textcolor[RGB]{0,120,0}{8.94$\downarrow$}  &     227.90     &  \textcolor[RGB]{0,120,0}{29.41$\downarrow$}  &     174.01     &  \textcolor[RGB]{0,120,0}{7.55$\downarrow$}  &     163.46     &  \textcolor[RGB]{128,128,128}{1.58$\downarrow$}  &     171.48     &  \textcolor[RGB]{0,120,0}{6.19$\downarrow$}  &     176.12     &  \textcolor[RGB]{0,120,0}{8.66$\downarrow$} \\

    CCRT$^*$     &     158.13     &     170.05     &  \textcolor[RGB]{0,120,0}{7.01$\downarrow$}  &     231.28     &  \textcolor[RGB]{0,120,0}{31.63$\downarrow$}  &     171.01     &  \textcolor[RGB]{0,120,0}{7.53$\downarrow$}  &     161.75     &  \textcolor[RGB]{0,120,0}{2.24$\downarrow$}  &     170.03     &  \textcolor[RGB]{0,120,0}{7.00$\downarrow$}  &     186.50     &  \textcolor[RGB]{0,120,0}{15.21$\downarrow$} \\

    ConAbl \cite{conabl}     &     160.81     &     177.41     &  \textcolor[RGB]{0,120,0}{9.36$\downarrow$}  &     226.12     &  \textcolor[RGB]{0,120,0}{28.88$\downarrow$}  &     172.44     &  \textcolor[RGB]{0,120,0}{6.74$\downarrow$}  &     162.28     &  \textcolor[RGB]{128,128,128}{0.91$\downarrow$}  &     172.13     &  \textcolor[RGB]{0,120,0}{6.58$\downarrow$}  &     178.94     &  \textcolor[RGB]{0,120,0}{10.13$\downarrow$} \\
    
    tuneSD-AB \cite{tuneSD}     &     161.68     &     175.51     &  \textcolor[RGB]{0,120,0}{7.88$\downarrow$}  &     227.91     &  \textcolor[RGB]{0,120,0}{29.06$\downarrow$}  &     173.28     &  \textcolor[RGB]{0,120,0}{6.69$\downarrow$}  &     164.29     &  \textcolor[RGB]{128,128,128}{1.59$\downarrow$}  &     172.53     &  \textcolor[RGB]{0,120,0}{6.29$\downarrow$}  &     181.29     &  \textcolor[RGB]{0,120,0}{10.82$\downarrow$} \\

    FMN \cite{Forget-me-not}     &     160.85     &     186.42     &  \textcolor[RGB]{0,120,0}{13.72$\downarrow$}  &     225.36     &  \textcolor[RGB]{0,120,0}{28.63$\downarrow$}  &     173.12     &  \textcolor[RGB]{0,120,0}{7.09$\downarrow$}  &     161.79     &  \textcolor[RGB]{128,128,128}{0.58$\downarrow$}  &     175.36     &  \textcolor[RGB]{0,120,0}{8.27$\downarrow$}  &     185.39     &  \textcolor[RGB]{0,120,0}{13.24$\downarrow$} \\

    Average & - & - & \textcolor[RGB]{0,120,0}{8.87$\downarrow$} & - & \textcolor[RGB]{0,120,0}{29.58$\downarrow$} & - & \textcolor[RGB]{0,120,0}{7.14$\downarrow$} & - & \textcolor[RGB]{128,128,128}{1.50$\downarrow$} & - & \textcolor[RGB]{0,120,0}{6.91$\downarrow$} & - & \textcolor[RGB]{0,120,0}{12.41$\downarrow$} \\
    \bottomrule
  \end{tabular}%
  }
\end{table*}

%% file: src/table-llm.tex
\begin{table*}[t]
  \centering
  \caption{Unlearning ``Porsche'' concept using different LLMs for comparision.}\label{tab:ablation-llm}
  \resizebox{0.7\linewidth}{!}{
  \begin{tabular}{l ccc ccc ccc}
    \toprule
    \multirow{2}{*}{\textbf{Method}} & \multicolumn{3}{c}{\textbf{CLIP}(U) $\downarrow$} & \multicolumn{3}{c}{\textbf{CLIP}(R) $\uparrow$} & \multicolumn{3}{c}{\textbf{FID} $\downarrow$} \\
    \cmidrule(lr){2-4} \cmidrule(lr){5-7} \cmidrule(lr){8-10}
    & \textbf{Orig} & \textbf{+} & \textbf{$\Delta$(\%)} & \textbf{Orig} & \textbf{+} & \textbf{$\Delta$(\%)} & \textbf{Orig} & \textbf{+} & \textbf{$\Delta$(\%)} \\
    \midrule
    \multicolumn{10}{c}{\textbf{LLM Choice: Gemini3.1 pro}} \\
    \midrule
    ESD \cite{esd} & 29.50 & 26.68 & {\color[RGB]{0,120,0}9.56 $\downarrow$} & 31.42 & 33.86 & {\color[RGB]{0,120,0}7.77 $\uparrow$} & 179.34 & \underline{164.77} & {\color[RGB]{0,120,0}8.12 $\downarrow$} \\
    tuneSD-l \cite{tuneSD} & 34.28 & 34.18 & {\color[RGB]{128,128,128}0.29 $\downarrow$} & 32.51 & 32.45 & {\color[RGB]{128,128,128}0.18 $\downarrow$} & 208.74 & 205.07 & {\color[RGB]{128,128,128}1.76 $\downarrow$} \\
    ResAlign \cite{ResAlign} & 31.69 & 31.61 & {\color[RGB]{128,128,128}0.25 $\downarrow$} & 32.01 & 33.28 & {\color[RGB]{0,120,0}3.97 $\uparrow$} & 188.23 & 181.36 & {\color[RGB]{0,120,0}3.65 $\downarrow$} \\
    CCRT \cite{CCRT} & 28.46 & \textbf{26.33} & {\color[RGB]{0,120,0}7.48 $\downarrow$} & 31.47 & 32.09 & {\color[RGB]{128,128,128}1.97 $\uparrow$} & 169.20 & \textbf{164.62} & {\color[RGB]{0,120,0}2.71 $\downarrow$} \\
    ConAbl \cite{conabl} & 30.79 & 30.50 & {\color[RGB]{128,128,128}0.94 $\downarrow$} & 31.57 & 31.89 & {\color[RGB]{128,128,128}1.01 $\uparrow$} & 179.28 & 178.95 & {\color[RGB]{128,128,128}0.18 $\downarrow$} \\
    FMN \cite{Forget-me-not} & 33.87 & 32.43 & {\color[RGB]{0,120,0}4.25 $\downarrow$} & 32.16 & \underline{34.79} & {\color[RGB]{0,120,0}8.18 $\uparrow$} & 207.05 & 202.03 & {\color[RGB]{0,120,0}2.42 $\downarrow$} \\
    tuneSD-u \cite{tuneSD} & 31.29 & 28.90 & {\color[RGB]{0,120,0}7.64 $\downarrow$} & 31.87 & \textbf{34.84} & {\color[RGB]{0,120,0}9.32 $\uparrow$} & 186.62 & 180.55 & {\color[RGB]{0,120,0}3.25 $\downarrow$} \\
    \textbf{\mysys (Ours)} & \underline{26.32} & - & - & \textbf{34.03} & - & - & \textbf{172.88} & - & - \\
    \textbf{Average} & - & - & {\color[RGB]{0,120,0}4.35 $\downarrow$} & - & - & {\color[RGB]{0,120,0}4.58 $\uparrow$} & - & - & {\color[RGB]{0,120,0}3.16 $\downarrow$} \\

    \midrule
    \multicolumn{10}{c}{\textbf{LLM Choice: DeepSeek-V3.2}} \\
    \midrule
    ESD \cite{esd} & 29.04 & \underline{27.11} & {\color[RGB]{0,120,0}6.65 $\downarrow$} & 31.21 & 31.60 & {\color[RGB]{128,128,128}1.25 $\uparrow$} & 162.58 & \textbf{159.42} & {\color[RGB]{128,128,128}1.94 $\downarrow$} \\
    tuneSD-l \cite{tuneSD} & 34.60 & 34.16 & {\color[RGB]{128,128,128}1.27 $\downarrow$} & 32.44 & \textbf{33.38} & {\color[RGB]{0,120,0}2.90 $\uparrow$} & 208.74 & 203.34 & {\color[RGB]{0,120,0}2.59 $\downarrow$} \\
    ResAlign \cite{ResAlign} & 31.06 & 30.81 & {\color[RGB]{128,128,128}0.80 $\downarrow$} & 31.62 & 32.17 & {\color[RGB]{128,128,128}1.74 $\uparrow$} & 163.36 & 161.88 & {\color[RGB]{128,128,128}0.91 $\downarrow$} \\
    CCRT \cite{CCRT} & 28.99 & 29.48 & {\color[RGB]{128,128,128}1.69 $\uparrow$} & 31.61 & \underline{32.84} & {\color[RGB]{0,120,0}3.89 $\uparrow$} & 161.43 & 162.05 & {\color[RGB]{128,128,128}0.38 $\uparrow$} \\
    ConAbl \cite{conabl} & 31.54 & 29.64 & {\color[RGB]{0,120,0}6.02 $\downarrow$} & 31.86 & 33.15 & {\color[RGB]{0,120,0}4.05 $\uparrow$} & 162.17 & 161.73 & {\color[RGB]{128,128,128}0.27 $\downarrow$} \\
    FMN \cite{Forget-me-not} & 33.63 & 32.44 & {\color[RGB]{0,120,0}3.54 $\downarrow$} & 31.95 & 32.71 & {\color[RGB]{0,120,0}2.38 $\uparrow$} & 204.72 & 200.77 & {\color[RGB]{128,128,128}1.93 $\downarrow$} \\
    tuneSD-u \cite{tuneSD} & 30.98 & 29.61 & {\color[RGB]{0,120,0}4.42 $\downarrow$} & 31.67 & 33.20 & {\color[RGB]{0,120,0}4.83 $\uparrow$} & 162.74 & \underline{160.80} & {\color[RGB]{128,128,128}1.19 $\downarrow$} \\
    \textbf{\mysys (Ours)} & \textbf{27.19} & - & - & \textbf{33.98} & - & - & \textbf{161.13} & - & - \\
    \textbf{Average} & - & - & {\color[RGB]{0,120,0}3.00 $\downarrow$} & - & - & {\color[RGB]{0,120,0}3.01 $\uparrow$} & - & - & {\color[RGB]{128,128,128}1.21 $\downarrow$} \\

    \midrule
    \multicolumn{10}{c}{\textbf{LLM Choice: GPT-5.2}} \\
    \midrule
    ESD \cite{esd} & 29.44 & \underline{26.66} & {\color[RGB]{0,120,0}9.44 $\downarrow$} & 31.47 & 32.05 & {\color[RGB]{128,128,128}1.84 $\uparrow$} & 162.20 & 162.69 & {\color[RGB]{128,128,128}0.30 $\uparrow$} \\
    tuneSD-l \cite{tuneSD} & 34.22 & 34.14 & {\color[RGB]{128,128,128}0.23 $\downarrow$} & 32.40 & \underline{33.47} & {\color[RGB]{0,120,0}3.30 $\uparrow$} & 165.12 & \underline{162.19} & {\color[RGB]{128,128,128}1.77 $\downarrow$} \\
    ResAlign \cite{ResAlign} & 31.51 & 31.21 & {\color[RGB]{128,128,128}0.95 $\downarrow$} & 31.96 & 32.15 & {\color[RGB]{128,128,128}0.59 $\uparrow$} & 163.95 & 162.82 & {\color[RGB]{128,128,128}0.69 $\downarrow$} \\
    CCRT \cite{CCRT} & 28.55 & \textbf{25.49} & {\color[RGB]{0,120,0}10.72 $\downarrow$} & 31.34 & 32.08 & {\color[RGB]{0,120,0}2.36 $\uparrow$} & 162.19 & 164.72 & {\color[RGB]{128,128,128}1.56 $\uparrow$} \\
    ConAbl \cite{conabl} & 31.22 & 30.68 & {\color[RGB]{128,128,128}1.73 $\downarrow$} & 31.92 & 32.37 & {\color[RGB]{128,128,128}1.41 $\uparrow$} & 163.21 & 163.63 & {\color[RGB]{128,128,128}0.26 $\uparrow$} \\
    FMN \cite{Forget-me-not} & 33.38 & 32.73 & {\color[RGB]{128,128,128}1.95 $\downarrow$} & 31.90 & \textbf{33.91} & {\color[RGB]{0,120,0}6.30 $\uparrow$} & 201.76 & 199.65 & {\color[RGB]{128,128,128}1.05 $\downarrow$} \\
    tuneSD-u \cite{tuneSD} & 31.08 & 29.88 & {\color[RGB]{0,120,0}3.86 $\downarrow$} & 31.77 & 33.22 & {\color[RGB]{0,120,0}4.56 $\uparrow$} & 160.47 & 164.02 & {\color[RGB]{120,0,0}2.21 $\uparrow$} \\
    \textbf{\mysys (Ours)} & \textbf{25.17} & - & - & \textbf{33.65} & - & - & \textbf{162.09} & - & - \\
    \textbf{Average} & - & - & {\color[RGB]{0,120,0}4.13 $\downarrow$} & - & - & {\color[RGB]{0,120,0}2.91 $\uparrow$} & - & - & {\color[RGB]{128,128,128}0.12 $\uparrow$} \\

    \midrule
    \multicolumn{10}{c}{\textbf{LLM Choice: Claude4.6}} \\
    \midrule
    ESD \cite{esd} & 28.69 & \underline{28.72} & {\color[RGB]{128,128,128}0.10 $\uparrow$} & 31.32 & 31.70 & {\color[RGB]{128,128,128}1.21 $\uparrow$} & 159.85 & 163.25 & {\color[RGB]{120,0,0}2.13 $\uparrow$} \\
    tuneSD-l \cite{tuneSD} & 34.24 & 33.20 & {\color[RGB]{0,120,0}3.04 $\downarrow$} & 32.44 & 32.29 & {\color[RGB]{128,128,128}0.46 $\downarrow$} & 164.91 & 163.94 & {\color[RGB]{128,128,128}0.59 $\downarrow$} \\
    ResAlign \cite{ResAlign} & 31.00 & 30.57 & {\color[RGB]{128,128,128}1.39 $\downarrow$} & 31.98 & 31.95 & {\color[RGB]{128,128,128}0.09 $\downarrow$} & 161.29 & 161.26 & {\color[RGB]{128,128,128}0.02 $\downarrow$} \\
    CCRT \cite{CCRT} & 28.56 & \textbf{28.52} & {\color[RGB]{128,128,128}0.14 $\downarrow$} & 31.40 & 31.62 & {\color[RGB]{128,128,128}0.70 $\uparrow$} & 161.81 & \underline{160.84} & {\color[RGB]{128,128,128}0.60 $\downarrow$} \\
    ConAbl \cite{conabl} & 31.53 & 30.02 & {\color[RGB]{0,120,0}4.79 $\downarrow$} & 31.95 & 32.06 & {\color[RGB]{128,128,128}0.34 $\uparrow$} & 161.42 & 163.47 & {\color[RGB]{128,128,128}1.27 $\uparrow$} \\
    FMN \cite{Forget-me-not} & 32.44 & 32.40 & {\color[RGB]{128,128,128}0.12 $\downarrow$} & 31.54 & \textbf{32.78} & {\color[RGB]{0,120,0}3.93 $\uparrow$} & 164.80 & 163.66 & {\color[RGB]{128,128,128}0.69 $\downarrow$} \\
    tuneSD-u \cite{tuneSD} & 31.22 & 29.61 & {\color[RGB]{0,120,0}5.16 $\downarrow$} & 31.93 & \underline{32.06} & {\color[RGB]{128,128,128}0.41 $\uparrow$} & 161.68 & \textbf{160.76} & {\color[RGB]{128,128,128}0.57 $\downarrow$} \\
    \textbf{\mysys (Ours)} & \textbf{28.18} & - & - & \textbf{32.98} & - & - & \textbf{163.56} & - & - \\
    \textbf{Average} & - & - & {\color[RGB]{0,120,0}2.08 $\downarrow$} & - & - & {\color[RGB]{128,128,128}0.86 $\uparrow$} & - & - & {\color[RGB]{128,128,128}0.13 $\uparrow$} \\

    \midrule
    \multicolumn{10}{c}{\textbf{LLM Choice: Qwen3.5-Plus}} \\
    \midrule
    ESD \cite{esd} & 29.23 & 28.15 & {\color[RGB]{0,120,0}3.69 $\downarrow$} & 31.43 & \underline{33.77} & {\color[RGB]{0,120,0}7.45 $\uparrow$} & 159.90 & \textbf{159.77} & {\color[RGB]{128,128,128}0.08 $\downarrow$} \\
    tuneSD-l \cite{tuneSD} & 34.31 & 34.14 & {\color[RGB]{128,128,128}0.50 $\downarrow$} & 32.41 & \textbf{33.88} & {\color[RGB]{0,120,0}4.54 $\uparrow$} & 165.28 & 163.12 & {\color[RGB]{128,128,128}1.31 $\downarrow$} \\
    ResAlign \cite{ResAlign} & 31.48 & 31.03 & {\color[RGB]{128,128,128}1.43 $\downarrow$} & 32.01 & 32.81 & {\color[RGB]{0,120,0}2.50 $\uparrow$} & 163.25 & 163.89 & {\color[RGB]{128,128,128}0.39 $\uparrow$} \\
    CCRT \cite{CCRT} & 27.25 & \textbf{26.69} & {\color[RGB]{0,120,0}2.06 $\downarrow$} & 30.84 & 32.78 & {\color[RGB]{0,120,0}6.29 $\uparrow$} & 160.57 & 162.18 & {\color[RGB]{128,128,128}1.00 $\uparrow$} \\
    ConAbl \cite{conabl} & 31.53 & 31.39 & {\color[RGB]{128,128,128}0.44 $\downarrow$} & 31.91 & 32.74 & {\color[RGB]{0,120,0}2.60 $\uparrow$} & 161.79 & 164.23 & {\color[RGB]{128,128,128}1.51 $\uparrow$} \\
    FMN \cite{Forget-me-not} & 33.85 & 32.31 & {\color[RGB]{0,120,0}4.55 $\downarrow$} & 32.19 & 33.40 & {\color[RGB]{0,120,0}3.76 $\uparrow$} & 164.15 & 163.78 & {\color[RGB]{128,128,128}0.23 $\downarrow$} \\
    tuneSD-u \cite{tuneSD} & 31.35 & \underline{30.48} & {\color[RGB]{0,120,0}2.78 $\downarrow$} & 31.79 & 33.10 & {\color[RGB]{0,120,0}4.12 $\uparrow$} & 162.30 & \underline{161.88} & {\color[RGB]{128,128,128}0.26 $\downarrow$} \\
    \textbf{\mysys (Ours)} & \textbf{25.33} & - & - & \textbf{33.28} & - & - & \textbf{162.79} & - & - \\
    \textbf{Average} & - & - & {\color[RGB]{0,120,0}2.21 $\downarrow$} & - & - & {\color[RGB]{0,120,0}4.46 $\uparrow$} & - & - & {\color[RGB]{128,128,128}0.15 $\uparrow$} \\
    \bottomrule
  \end{tabular}
  }
\end{table*}

%% file: src/table-hyper.tex
\begin{table*}[t]
  \centering
  \caption{Comparison of unlearning performance under different K when unlearning the ``Porsche'' concept as an example.}\label{tab:ablation-hyper}
  \resizebox{0.8\linewidth}{!}{%
  \begin{tabular}{l ccc ccc ccc}
    \toprule
    \multirow{2}{*}{\textbf{Method}} & \multicolumn{3}{c}{\textbf{CLIP}(U) $\downarrow$} & \multicolumn{3}{c}{\textbf{CLIP}(R) $\uparrow$} & \multicolumn{3}{c}{\textbf{FID} $\downarrow$} \\ 
    \cmidrule(lr){2-4} \cmidrule(lr){5-7} \cmidrule(lr){8-10}
    & \textbf{Orig} & \textbf{+} & \textbf{$\Delta$(\%)} & \textbf{Orig} & \textbf{+} & \textbf{$\Delta$(\%)} & \textbf{Orig} & \textbf{+} & \textbf{$\Delta$(\%)} \\
    \midrule
    \multicolumn{10}{c}{\textbf{K=5}} \\
    \midrule
    ESD \cite{esd} & 28.90 & \textbf{29.10} & {\color[RGB]{128,128,128}0.69 $\uparrow$} & 32.26 & 32.29 & {\color[RGB]{128,128,128}0.09 $\uparrow$} & 176.04 & 180.08 & {\color[RGB]{120,0,0}2.29 $\uparrow$} \\
    tuneSD-l \cite{tuneSD} & 34.15 & 33.92 & {\color[RGB]{128,128,128}0.67 $\downarrow$} & 34.44 & \textbf{34.78} & {\color[RGB]{128,128,128}0.99 $\uparrow$} & 184.64 & 192.23 & {\color[RGB]{120,0,0}4.11 $\uparrow$} \\
    ResAlign \cite{ResAlign} & 30.98 & 31.50 & {\color[RGB]{128,128,128}1.68 $\uparrow$} & 33.29 & 33.71 & {\color[RGB]{128,128,128}1.26 $\uparrow$} & 176.40 & 180.01 & {\color[RGB]{120,0,0}2.05 $\uparrow$} \\
    CCRT \cite{CCRT} & 28.67 & \underline{27.99} & {\color[RGB]{0,120,0}2.37 $\downarrow$} & 32.72 & 33.36 & {\color[RGB]{128,128,128}1.96 $\uparrow$} & 182.74 & 179.96 & {\color[RGB]{128,128,128}1.52 $\downarrow$} \\
    ConAbl \cite{conabl} & 31.41 & 31.45 & {\color[RGB]{128,128,128}0.13 $\uparrow$} & 33.28 & \underline{33.80} & {\color[RGB]{128,128,128}1.56 $\uparrow$} & 179.68 & 182.24 & {\color[RGB]{128,128,128}1.42 $\uparrow$} \\
    FMN \cite{Forget-me-not} & 32.11 & 31.96 & {\color[RGB]{128,128,128}0.47 $\downarrow$} & 32.66 & 33.67 & {\color[RGB]{0,120,0}3.09 $\uparrow$} & 179.78 & 186.96 & {\color[RGB]{120,0,0}3.99 $\uparrow$} \\
    tuneSD-u \cite{tuneSD} & 31.00 & 30.59 & {\color[RGB]{128,128,128}1.32 $\downarrow$} & 33.28 & 33.24 & {\color[RGB]{128,128,128}0.12 $\downarrow$} & 177.48 & \textbf{171.74} & {\color[RGB]{0,120,0}3.23 $\downarrow$} \\
    \textbf{\mysys (Ours)} & {29.32} & - & - & 33.52 & - & - & \underline{181.92} & - & - \\
    \textbf{Average} & - & - & {\color[RGB]{128,128,128}0.33 $\downarrow$} & - & - & {\color[RGB]{128,128,128}1.26 $\uparrow$} & - & - & {\color[RGB]{128,128,128}1.30 $\uparrow$} \\

    \midrule
    \multicolumn{10}{c}{\textbf{K=10}} \\
    \midrule
    ESD \cite{esd} & 29.38 & 28.86 & {\color[RGB]{128,128,128}1.77 $\downarrow$} & 32.52 & 32.16 & {\color[RGB]{128,128,128}1.11 $\downarrow$} & 172.00 & \textbf{167.36} & {\color[RGB]{0,120,0}2.70 $\downarrow$} \\
    tuneSD-l \cite{tuneSD} & 34.21 & 34.11 & {\color[RGB]{128,128,128}0.29 $\downarrow$} & 34.58 & \underline{34.64} & {\color[RGB]{128,128,128}0.17 $\uparrow$} & 188.00 & 188.40 & {\color[RGB]{128,128,128}0.21 $\uparrow$} \\
    ResAlign \cite{ResAlign} & 31.51 & 32.06 & {\color[RGB]{128,128,128}1.75 $\uparrow$} & 33.82 & 33.75 & {\color[RGB]{128,128,128}0.21 $\downarrow$} & 182.46 & 181.35 & {\color[RGB]{128,128,128}0.61 $\downarrow$} \\
    CCRT \cite{CCRT} & 27.54 & \textbf{24.42} & {\color[RGB]{0,120,0}11.33 $\downarrow$} & 32.79 & 32.83 & {\color[RGB]{128,128,128}0.12 $\uparrow$} & 180.87 & 178.98 & {\color[RGB]{128,128,128}1.04 $\downarrow$} \\
    ConAbl \cite{conabl} & 31.46 & 32.06 & {\color[RGB]{128,128,128}1.91 $\uparrow$} & 33.02 & 33.75 & {\color[RGB]{0,120,0}2.21 $\uparrow$} & 175.24 & 178.48 & {\color[RGB]{128,128,128}1.85 $\uparrow$} \\
    FMN \cite{Forget-me-not} & 32.35 & 32.08 & {\color[RGB]{128,128,128}0.83 $\downarrow$} & 33.27 & \textbf{34.84} & {\color[RGB]{0,120,0}4.72 $\uparrow$} & 190.17 & 186.29 & {\color[RGB]{0,120,0}2.04 $\downarrow$} \\
    tuneSD-u \cite{tuneSD} & 31.39 & \underline{30.55} & {\color[RGB]{0,120,0}2.68 $\downarrow$} & 33.24 & 33.03 & {\color[RGB]{128,128,128}0.63 $\downarrow$} & 178.28 & 175.37 & {\color[RGB]{128,128,128}1.63 $\downarrow$} \\
    \textbf{\mysys (Ours)} & 30.23 & - & - & 33.92 & - & - & \underline{171.33} & - & - \\
    \textbf{Average} & - & - & {\color[RGB]{128,128,128}1.89 $\downarrow$} & - & - & {\color[RGB]{128,128,128}0.75 $\uparrow$} & - & - & {\color[RGB]{128,128,128}0.85 $\downarrow$} \\

    \midrule
    \multicolumn{10}{c}{\textbf{K=20}} \\
    \midrule
    ESD \cite{esd} & 29.50 & 26.68 & {\color[RGB]{0,120,0}9.56 $\downarrow$} & 31.42 & 33.86 & {\color[RGB]{0,120,0}7.77 $\uparrow$} & 179.34 & \underline{164.77} & {\color[RGB]{0,120,0}8.12 $\downarrow$} \\
    tuneSD-l \cite{tuneSD} & 34.28 & 34.18 & {\color[RGB]{128,128,128}0.29 $\downarrow$} & 32.51 & 32.45 & {\color[RGB]{128,128,128}0.18 $\downarrow$} & 208.74 & 205.07 & {\color[RGB]{128,128,128}1.76 $\downarrow$} \\
    ResAlign \cite{ResAlign} & 31.69 & 31.61 & {\color[RGB]{128,128,128}0.25 $\downarrow$} & 32.01 & 33.28 & {\color[RGB]{0,120,0}3.97 $\uparrow$} & 188.23 & 181.36 & {\color[RGB]{0,120,0}3.65 $\downarrow$} \\
    CCRT \cite{CCRT} & 28.46 & \underline{26.33} & {\color[RGB]{0,120,0}7.48 $\downarrow$} & 31.47 & 32.09 & {\color[RGB]{128,128,128}1.97 $\uparrow$} & 169.20 & \textbf{164.62} & {\color[RGB]{0,120,0}2.71 $\downarrow$} \\
    ConAbl \cite{conabl} & 30.79 & 30.50 & {\color[RGB]{128,128,128}0.94 $\downarrow$} & 31.57 & 31.89 & {\color[RGB]{128,128,128}1.01 $\uparrow$} & 179.28 & 178.95 & {\color[RGB]{128,128,128}0.18 $\downarrow$} \\
    FMN \cite{Forget-me-not} & 33.87 & 32.43 & {\color[RGB]{0,120,0}4.25 $\downarrow$} & 32.16 & \underline{34.79} & {\color[RGB]{0,120,0}8.18 $\uparrow$} & 207.05 & 202.03 & {\color[RGB]{0,120,0}2.42 $\downarrow$} \\
    tuneSD-u \cite{tuneSD} & 31.29 & 28.90 & {\color[RGB]{0,120,0}7.64 $\downarrow$} & 31.87 & \textbf{34.84} & {\color[RGB]{0,120,0}9.32 $\uparrow$} & 186.62 & 180.55 & {\color[RGB]{0,120,0}3.25 $\downarrow$} \\
    \textbf{\mysys (Ours)} & \textbf{26.32} & - & - & 34.03 & - & - & 172.88 & - & - \\
    \textbf{Average} & - & - & {\color[RGB]{0,120,0}4.35 $\downarrow$} & - & - & {\color[RGB]{0,120,0}4.58 $\uparrow$} & - & - & {\color[RGB]{0,120,0}3.16 $\downarrow$} \\

    \midrule
    \multicolumn{10}{c}{\textbf{K=30}} \\
    \midrule
    ESD \cite{esd} & 29.36 & 27.35 & {\color[RGB]{0,120,0}6.85 $\downarrow$} & 31.99 & 33.43 & {\color[RGB]{0,120,0}4.50 $\uparrow$} & 171.14 & 173.10 & {\color[RGB]{128,128,128}1.15 $\uparrow$} \\
    tuneSD-l \cite{tuneSD} & 34.29 & 34.16 & {\color[RGB]{128,128,128}0.38 $\downarrow$} & 34.48 & {34.71} & {\color[RGB]{128,128,128}0.67 $\uparrow$} & 189.36 & 191.72 & {\color[RGB]{128,128,128}1.25 $\uparrow$} \\
    ResAlign \cite{ResAlign} & 31.22 & 29.35 & {\color[RGB]{0,120,0}5.99 $\downarrow$} & 33.11 & 34.21 & {\color[RGB]{0,120,0}3.32 $\uparrow$} & 172.84 & 173.28 & {\color[RGB]{128,128,128}0.25 $\uparrow$} \\
    CCRT \cite{CCRT} & 28.74 & \underline{25.82} & {\color[RGB]{0,120,0}10.16 $\downarrow$} & 32.78 & {34.71} & {\color[RGB]{0,120,0}5.89 $\uparrow$} & 181.12 & \textbf{169.97} & {\color[RGB]{0,120,0}6.16 $\downarrow$} \\
    ConAbl \cite{conabl} & 31.63 & {29.34} & {\color[RGB]{0,120,0}7.24 $\downarrow$} & 33.31 & 34.46 & {\color[RGB]{0,120,0}3.45 $\uparrow$} & 179.48 & 175.53 & {\color[RGB]{0,120,0}2.20 $\downarrow$} \\
    FMN \cite{Forget-me-not} & 32.73 & 32.13 & {\color[RGB]{128,128,128}1.83 $\downarrow$} & 33.75 & \textbf{34.90} & {\color[RGB]{0,120,0}3.41 $\uparrow$} & 190.08 & 184.66 & {\color[RGB]{0,120,0}2.85 $\downarrow$} \\
    tuneSD-u \cite{tuneSD} & 31.28 & 29.49 & {\color[RGB]{0,120,0}5.72 $\downarrow$} & 33.34 & 34.37 & {\color[RGB]{0,120,0}3.09 $\uparrow$} & 180.61 & 180.13 & {\color[RGB]{128,128,128}0.27 $\downarrow$} \\
    \textbf{\mysys (Ours)} & \textbf{25.32} & - & - & \underline{34.87} & - & - & \underline{168.30} & - & - \\
    \textbf{Average} & - & - & {\color[RGB]{0,120,0}5.45 $\downarrow$} & - & - & {\color[RGB]{0,120,0}3.48 $\uparrow$} & - & - & {\color[RGB]{128,128,128}1.26 $\downarrow$} \\
    \bottomrule
  \end{tabular}%
  }
\end{table*}